\documentclass[11pt,a4paper]{article}

\usepackage{pgfplots}
\usepackage{times}
\usepgfplotslibrary{groupplots}
\usepgfplotslibrary{fillbetween}
\usetikzlibrary{intersections}
\pgfplotsset{compat=1.18}
\usepackage{hyperref}
\hypersetup{
    colorlinks=false,
    pdfborder={0 0 0}
}

\definecolor{umtam}{RGB}{0,90,156}   
\definecolor{mofasgd}{RGB}{204,85,0} 
\definecolor{galore}{RGB}{34,139,34} 

\usepackage[utf8]{inputenc}
\usepackage[T1]{fontenc}
\usepackage[english]{babel}
\usepackage{amsmath,amssymb,amsthm}
\newtheorem{definition}{Definition}
\usepackage{graphicx}
\usepackage[margin=1in]{geometry}
\usepackage{setspace}
\usepackage[numbers,sort&compress]{natbib}
\usepackage{booktabs}
\usepackage{enumitem}
\usepackage{xcolor}
\usepackage{algorithm}
\usepackage{float}
\usepackage{algpseudocode}
\newtheorem{theorem}{Theorem}[section]
\newtheorem{lemma}[theorem]{Lemma}

\theoremstyle{definition}

\theoremstyle{remark}

\usepackage{multirow}
\usepackage{authblk}

\usetikzlibrary{shapes.geometric, arrows.meta, positioning, fit, backgrounds, calc, shadows}

\definecolor{lightblue}{RGB}{220,235,250}
\definecolor{lightorange}{RGB}{255,235,220}
\definecolor{lightgreen}{RGB}{220,245,220}
\definecolor{lightgray}{RGB}{245,245,245}
\definecolor{darkgray}{RGB}{80,80,80}

\title{Bridging Training and Merging Through Momentum-Aware Optimization}

\author[1]{Alireza Moayedikia\thanks{Corresponding author. Email: amoayedikia@swin.edu.au}}
\author[2]{Alicia Troncoso}

\affil[1]{Swinburne Business School, Swinburne University of Technology, Australia \vspace{0.3em}}
\affil[2]{Data Science \& Big Data Lab, Universidad Pablo de Olavide, Seville, Spain \vspace{0.3em}}

\date{}

\begin{document}
\sloppy

\maketitle

\begin{abstract}
Training large neural networks and merging task-specific models both exploit low-rank structure and require parameter importance estimation, yet these challenges have been pursued in isolation. Current workflows compute curvature information during training, discard it, then recompute similar information for merging---wasting computation and discarding valuable trajectory data. We introduce a unified framework that maintains factorized momentum and curvature statistics during training, then reuses this information for geometry-aware model composition. \textcolor{black}{The proposed method incurs modest memory overhead (approximately 30\% over AdamW) to accumulate task saliency scores that enable curvature-aware merging. These scores, computed as a byproduct of optimization, provide importance estimates comparable to post-hoc Fisher computation while producing merge-ready models directly from training.} We establish convergence guarantees for non-convex objectives with approximation error bounded by gradient singular value decay. On natural language understanding benchmarks, curvature-aware parameter selection outperforms magnitude-only baselines across all sparsity levels, with multi-task merging improving 1.6\% over strong baselines. The proposed framework exhibits rank-invariant convergence and superior hyperparameter robustness compared to existing low-rank optimizers. \textcolor{black}{By treating the optimization trajectory as a reusable asset rather than discarding it, our approach demonstrates that training-time curvature information suffices for effective model composition, enabling a unified training-merging pipeline.} \\

\end{abstract}

\section{Introduction}
\label{sec:introduction}

Training a 7B parameter language model with Adam requires over 80GB of memory for optimizer states alone---more than double the memory needed for the model parameters themselves. This memory burden has driven extensive research into efficient training methods that reduce optimizer footprint through low-rank projections, quantization, or stateless updates. Separately, the growing practice of fine-tuning foundation models for specialized tasks has created a parallel challenge: combining multiple task-specific models into a single unified model without costly retraining. Both problems have attracted substantial research attention, yet they have been pursued largely in isolation.

Memory-efficient training methods such as LoRA~\citep{hu2021lora}, GaLore~\citep{zhao2024galore}, and APOLLO~\citep{zhu2024apollo} exploit the observation that gradients and optimizer states exhibit low intrinsic dimensionality, enabling compression through low-rank factorization. These methods track gradient subspaces, accumulate momentum in compressed form, and maintain curvature estimates that guide optimization. Model merging methods such as TIES~\citep{yadav2023ties}, DARE~\citep{yu2023dare}, and Fisher-weighted averaging~\citep{matena2022merging} address a different challenge: identifying which parameters matter for each task and resolving conflicts when combining models trained on different objectives. These methods analyze parameter magnitudes, compute importance scores, and apply interference resolution heuristics.

Despite solving different problems, these two research directions share fundamental computational structure. Both exploit low-rank properties of neural network optimization. Both require estimates of parameter importance or sensitivity. Both benefit from curvature information that captures the local geometry of the loss landscape. Yet in practice, training methods compute this information, use it briefly, and discard it upon completion. Merging methods then recompute similar information from scratch---running additional forward and backward passes to estimate Fisher information, analyzing final parameter values to infer importance, or applying heuristics that approximate what the training process already knew. This sequential workflow wastes computation and discards valuable trajectory information that could inform more principled model composition.

This observation motivates several questions. Can training and merging be optimized jointly rather than sequentially? How much training-time computation can be reused for merging, and with what theoretical guarantees? Does the optimization trajectory---not just final parameters, but how they evolved---enable more effective model composition? Can we provide unified theoretical analysis spanning both phases?

We propose UMTAM (Unified Momentum-Trajectory Aware Training and Merging), a framework that addresses memory-efficient training and principled model merging through shared computational structure. UMTAM maintains a factorized representation of momentum during training that exploits low-rank gradient structure, achieving memory efficiency comparable to state-of-the-art methods. Rather than discarding this information after training, UMTAM continuously accumulates task-specific saliency scores that weight each parameter by both its deviation from initialization and the local curvature of the loss landscape. \textcolor{black}{These scores, computed as a byproduct of optimization with modest additional memory overhead (approximately $K \cdot mn \cdot k/100$ for $K$ tasks at sparsity $k\%$)}, enable geometry-aware model merging that respects each task's loss surface.

\textcolor{black}{The core advantage of UMTAM lies in unifying the conventional sequential workflow into a single integrated pipeline. Current practice follows a two-phase approach: (1) train with a memory-efficient optimizer such as GaLore or LoRA, (2) discard all optimizer states upon completion, then (3) compute post-hoc importance scores via Fisher information or magnitude analysis, and finally (4) merge task-specific models. UMTAM demonstrates that this separation is unnecessary: training-time curvature information provides importance estimates comparable to post-hoc Fisher computation. By accumulating saliency scores during optimization, UMTAM produces merge-ready models directly from training. This integration provides three concrete benefits: architectural simplification (eliminating the separate importance estimation pipeline), richer information (capturing how parameters evolved during optimization, not just their final values), and reduced workflow complexity (optimizer states flow directly to merging rather than requiring separate analysis).}

The framework introduces three interconnected components.

First, dual momentum factorization with error feedback maintains accurate gradient estimates despite low-rank compression, preserving convergence properties while reducing memory requirements. Second, factorized second-order statistics enable adaptive preconditioning using only $O(m + n)$ memory per $m \times n$ weight matrix, compared to $O(mn)$ for full preconditioners. Third, curvature-aware merging reuses the accumulated momentum and curvature information to perform principled model composition, identifying parameter conflicts through importance-weighted sign election and resolving them using trajectory information. Our contributions are:

\begin{enumerate}
    \item \textcolor{black}{\textbf{Unified training-merging framework.} UMTAM bridges memory-efficient training and model merging through shared low-rank momentum factorization. Training-time curvature information is preserved and reused for composition, demonstrating that trajectory-aware optimization produces importance estimates sufficient for principled model merging.}
    
    \item \textbf{Convergence guarantees.} We prove that UMTAM achieves $O(1/\sqrt{T})$ convergence for non-convex objectives despite dual momentum tracking, with approximation error bounded by the singular value decay of the true gradient. For strongly convex objectives, we establish linear convergence rates.

    \item \textcolor{black}{\textbf{Trajectory-informed saliency.} Unlike post-hoc methods that infer importance solely from final parameters, UMTAM accumulates saliency throughout optimization via the lightweight approximation in Equation~\eqref{eq:saliency}, which weights parameters by both their cumulative squared deviation from initialization and local curvature estimates. While not a complete trajectory record, this provides richer information for merging than magnitude-based or static curvature approaches.}
    
    \item \textbf{Empirical validation.} On multi-task merging benchmarks, UMTAM's curvature-aware approach outperforms magnitude-only baselines across all sparsity levels, with largest gains under aggressive parameter budgets. Comparative analysis against GaLore and MoFaSGD demonstrates competitive training performance with superior hyperparameter robustness.
\end{enumerate}

Our theoretical analysis establishes several properties beyond convergence. We provide approximation quality guarantees for the factorized preconditioner, relating its error to singular value decay of the Hessian. We derive merging quality bounds quantifying expected loss on merged models in terms of task diversity and sparsity. We establish PAC-Bayes generalization bounds accounting for the low-rank constraint and provide multi-task guarantees for merged models.

\textcolor{black}{By reusing training-time curvature information for merging, UMTAM demonstrates that trajectory-informed optimization provides importance estimates comparable to post-hoc methods, enabling a unified pipeline that simplifies the practitioner's workflow.} The curvature-aware merging respects local geometry of each task's loss landscape, enabling more effective balancing of multiple objectives than naive averaging or magnitude-based methods. The framework operates within memory budgets comparable to existing efficient training methods while adding principled multi-task composition capability. \textcolor{black}{To facilitate reproducibility, we release our implementation at \url{https://github.com/amoayedikia/UMTAM}.}

The remainder of this paper is organized as follows. Section~\ref{sec:related} reviews memory-efficient training and model merging methods, examining what information each computes and discards. Section~\ref{sec:umtam} presents the UMTAM framework, including algorithms and theoretical analysis. Section~\ref{sec:experiments} evaluates UMTAM on training benchmarks and multi-task merging scenarios. Section~\ref{sec:conclusion} discusses implications and future directions.

\section{Related Work}
\label{sec:related}

The challenges addressed by UMTAM---memory-efficient training of large neural networks and effective composition of specialized models---have been pursued largely along independent research tracks. This section reviews both domains, examining not only the technical contributions of existing methods but also the information they compute, preserve, or discard. Understanding these patterns reveals opportunities for unification that motivate our framework.

\subsection{Memory-Efficient Training}

Modern neural network optimization confronts a significant memory bottleneck. For a model with $n$ parameters trained using adaptive methods like Adam, the optimizer states alone require approximately $8n$ to $12n$ bytes of memory in standard precision---often exceeding the memory required for the model parameters themselves. This challenge has driven diverse strategies for memory reduction, spanning low-rank factorization, gradient compression, optimizer state compression, and second-order preconditioning.

The observation that parameter updates during fine-tuning exhibit low intrinsic dimensionality has given rise to parameter-efficient training methods. Low-Rank Adaptation (LoRA)~\citep{hu2021lora} represents weight updates as $\Delta W = BA$ where $B \in \mathbb{R}^{d \times r}$ and $A \in \mathbb{R}^{r \times k}$ with rank $r \ll \min(d,k)$, reducing trainable parameters to as little as $0.01\%$ of the base model while achieving performance comparable to full fine-tuning. Subsequent refinements addressed various limitations: rsLoRA~\citep{kalajdzievski2023rslora} corrected the scaling factor to $\alpha/\sqrt{r}$ for stable training at higher ranks, DoRA~\citep{liu2024dora} decomposed weights into magnitude and direction components, and AdaLoRA~\citep{zhang2023adalora} introduced dynamic rank allocation based on layer importance. ReLoRA~\citep{lialin2023relora} extended low-rank methods to pre-training by periodically merging adapters into base weights, demonstrating that sequential low-rank updates can approximate full-rank optimization. These methods discover and exploit low-dimensional structure during training, yet this structural information is typically discarded when final parameters are saved.

A complementary line of work applies low-rank projection directly to gradients rather than weights, enabling full-parameter learning while compressing optimizer states. GaLore~\citep{zhao2024galore} projects gradients into a low-rank subspace via SVD, achieving up to $65.5\%$ reduction in optimizer state memory while matching full-rank training performance on LLaMA architectures up to 7B parameters. LDAdam~\citep{robert2024ldadam} improved upon this through projection-aware updates and dual compression. Flora~\citep{hao2024flora} revealed that LoRA implicitly performs gradient compression, suggesting that weight-space and gradient-space methods may be unified through appropriate factorization. APOLLO~\citep{zhu2024apollo} pushed compression further, achieving SGD-level memory costs with AdamW-level performance through channel-wise scaling, demonstrating viability even at rank-1 constraints. These gradient-based methods compute valuable information about which directions in parameter space matter most, yet this trajectory information guides only training before being discarded.

Optimizer states themselves represent a substantial compression target. Adafactor~\citep{shazeer2018adafactor} pioneered factorized second-moment estimation, storing only row and column statistics rather than the full matrix, reducing memory from $O(mn)$ to $O(m+n)$. CAME~\citep{luo2023came} addressed stability issues through confidence-guided updates. Quantization provides an orthogonal strategy: 8-bit Adam~\citep{dettmers2021eightbit} achieved $75\%$ memory reduction through block-wise quantization, while QLoRA~\citep{dettmers2024qlora} combined 4-bit quantization with LoRA to enable fine-tuning of 65B models on consumer hardware. These methods explicitly track second-order statistics characterizing loss landscape curvature---information that guides adaptive learning rates but is discarded after training, despite being precisely what Fisher-weighted merging methods require.

Second-order methods exploit curvature information more directly for faster convergence. Shampoo~\citep{gupta2018shampoo} uses Kronecker product approximations to construct structure-aware preconditioners, requiring $O(m^2 + n^2)$ memory rather than the prohibitive $O(m^2n^2)$ of full-matrix methods while achieving 2--3$\times$ faster convergence. KFAC~\citep{martens2015kfac} approximates the Fisher Information Matrix for natural gradient descent. SOAP~\citep{vyas2024soap} established that Shampoo is equivalent to Adafactor in a rotated eigenbasis, then replaced Adafactor with full Adam in that space, achieving $40\%$ fewer iterations than AdamW. This connection between second-order preconditioning and first-order adaptive methods suggests that curvature information can be maintained compactly through appropriate basis transformations---an insight central to our framework.

At the other extreme, stateless optimizers eliminate optimizer memory entirely. Lion~\citep{chen2023lion} uses only momentum tracking without second moments, achieving $50\%$ memory reduction. SWAN~\citep{ma2024swan} preprocesses gradients through stateless normalization, demonstrating $2\times$ speedup on LLaMA pre-training. System-level approaches like ZeRO~\citep{rajbhandari2020zero} partition states across devices, while LOMO~\citep{lv2023lomo} fuses gradient computation with updates to eliminate gradient storage. While effective for training in isolation, these approaches discard the optimization trajectory that could inform subsequent model composition.

\subsection{Model Merging}

The proliferation of task-specific fine-tuned models has created demand for methods that combine capabilities without expensive retraining. The central question is whether parameters trained on different tasks can be meaningfully combined in weight space. Approaches range from simple averaging justified by loss landscape geometry, through compositional methods that treat updates as algebraic objects, to principled Bayesian formulations that weight parameters by uncertainty.

The theoretical foundation for model merging rests on mode connectivity. Research by Garipov et al.~\citep{garipov2018loss} and Draxler et al.~\citep{draxler2018essentially} established that optima found by SGD from shared initialization are connected by low-loss paths. Frankle et al.~\citep{frankle2020linear} connected this to the Lottery Ticket Hypothesis, showing models trained from the same starting point exhibit strong linear connectivity. Building on this geometric understanding, Model Soups~\citep{wortsman2022model} demonstrated that averaging models fine-tuned with different hyperparameters from the same pretrained checkpoint yields enhanced performance and robustness. Stochastic Weight Averaging~\citep{izmailov2018swa} averages trajectory snapshots to find flatter minima. These methods establish that simple averaging can be effective when models remain in connected loss regions, though they do not address conflicting objectives across different tasks.

Task arithmetic~\citep{ilharco2022editing} enables more flexible composition by treating model updates as vectors amenable to algebraic manipulation. Defining task vectors as $\tau_t = \theta_t - \theta_0$, capabilities can be combined through addition $\theta_{\text{multi}} = \theta_0 + \sum \lambda_i \tau_i$, removed through negation, or scaled to adjust influence. This achieves $60$--$70\%$ normalized accuracy across vision tasks without accessing training data. However, performance degrades when combining multiple tasks due to sign conflicts---where different models push parameters in opposing directions---and redundancy, where small magnitude changes add noise. This degradation, typically $5$--$10\%$ when merging five or more tasks, motivates interference resolution methods.

TIES-Merging~\citep{yadav2023ties} addresses interference through trimming to retain only high-magnitude parameters, electing signs by weighted majority vote, and merging only agreeing parameters, achieving $2$--$4\%$ improvement over task arithmetic. DARE~\citep{yu2023dare} discovered that fine-tuning deltas exhibit extreme redundancy, with $90$--$99\%$ of parameters droppable without performance loss in larger models; combining both approaches yields the DARE-TIES baseline. AdaMerging~\citep{yang2024adamerging} learns merging coefficients through entropy minimization, while Model Breadcrumbs~\citep{davari2024breadcrumbs} employs dual masking to handle both outliers and noise. These heuristic methods achieve strong empirical performance but determine parameter importance solely from final values, without access to optimization trajectory information that could reveal why parameters changed.

Bayesian approaches provide a principled alternative by framing merging as probabilistic inference. Fisher-weighted averaging~\citep{matena2022merging} weights parameters by Fisher information $F_i^{(j)}$, which characterizes uncertainty: $\theta_{\text{merged}}^{(j)} = {\sum_i F_i^{(j)} \lambda_i \theta_i^{(j)}}/{\sum_i F_i^{(j)} \lambda_i}$, where high Fisher information indicates parameters to preserve carefully. RegMean~\citep{jin2023regmean} formulates merging as minimizing prediction discrepancies using activation statistics. Recent work~\citep{daheim2023unified} unified these perspectives, showing both methods project into task-specific subspaces before combining---similar to how low-rank training operates in gradient subspaces. Concurrent work on OTA-Merging~\citep{mahdavinia2025ota} demonstrated that Adam's second-moment statistics can serve as an effective curvature proxy for both parameter selection and weighted aggregation. \textcolor{black}{However, OTA-Merging relies on static second-order moments captured at training completion, whereas UMTAM accumulates trajectory-aware saliency throughout optimization---weighting parameters by both their cumulative deviation from initialization and local curvature. This distinction enables UMTAM to capture \emph{how} parameters evolved, not merely their final importance.}

A key observation emerges from this review: Fisher merging requires computing curvature information specifically for merging, yet second-order training methods already compute equivalent approximations during optimization. This redundancy---computing curvature twice for different purposes---exemplifies the broader pattern motivating our work. Training methods compute rich information about parameter importance, gradient structure, and loss landscape geometry, then discard it when training completes. Merging methods must subsequently rediscover similar information through independent analysis. GaLore discovers low-rank gradient structure; TIES independently identifies important parameters through magnitude trimming. Shampoo estimates Fisher approximations for preconditioning; Fisher merging recomputes this for parameter weighting. UMTAM addresses this inefficiency by maintaining a unified representation that serves both memory-efficient optimization and principled model composition, reusing the trajectory and curvature information that current approaches compute and discard. \textcolor{black}{Table~\ref{tab:method_comparison} summarizes these distinctions, comparing training and merging methods with respect to memory requirements, curvature handling, and trajectory preservation.}

\textcolor{black}{
\begin{table}[H]
\centering
\caption{Comparison of memory-efficient training and merging methods. UMTAM uniquely preserves both curvature and trajectory information during training, enabling direct reuse for model composition without separate importance estimation.}
\label{tab:method_comparison}
\renewcommand{\arraystretch}{1.15}
\begin{tabular}{lccc}
\toprule
\textbf{Method} & \textbf{Training Memory} & \textbf{Curvature} & \textbf{Trajectory} \\
\midrule
AdamW & $O(mn)$ & Computed, discarded & $\times$ \\
GaLore & $O(mr + nr)$ & $\times$ & $\times$ \\
Shampoo & $O(m^2 + n^2)$ & Computed, discarded & $\times$ \\
MoFaSGD & $O(mr + nr)$ & Partial & $\times$ \\
TIES/DARE & --- & Post-hoc (magnitude) & $\times$ \\
Fisher Merging & --- & Post-hoc (recomputed) & $\times$ \\
OTA-Merging & --- & Reused from Adam & $\times$ \\
\midrule
\textbf{UMTAM} & $O(mr + nr + m + n)$ & \checkmark Preserved & \checkmark Preserved \\
\bottomrule
\end{tabular}
\end{table}}

\section{Unified Momentum-Trajectory Aware Training and Merging} 
\label{sec:umtam}
This section presents UMTAM, a novel framework that simultaneously addresses two critical challenges in modern deep learning: memory-efficient training of large neural networks and principled composition of task-specific models. The first challenge arises from the memory bottleneck imposed by optimizer states, which often consume 2-3$\times$ more memory than the model parameters themselves, limiting the scale of models that can be trained on available hardware. The second challenge concerns the efficient integration of capabilities from multiple fine-tuned models without requiring expensive retraining from scratch, enabling flexible multi-task deployment while preserving task-specific expertise.

Our approach unifies these problems through a common theoretical foundation based on low-rank structure in optimization trajectories and curvature-aware geometric reasoning. During training, UMTAM maintains a compact factorized representation of momentum that exploits the intrinsic low dimensionality of gradient spaces, coupled with factorized second-order statistics that enable adaptive preconditioning. The same momentum and curvature information accumulated during training is then reused during the merging phase to perform geometry-aware model composition, ensuring that the merged model respects the local loss landscape around each task's optimum.

The section is organized into two main parts. The first subsection introduces the UMTAM framework, presenting the algorithm design, mathematical formulation, and implementation details for both the training and merging phases. Readers will learn how dual momentum factorization achieves memory efficiency, how error feedback ensures convergence despite compression, and how curvature-aware merging resolves conflicts between task-specific models. The second subsection provides comprehensive theoretical analysis, establishing convergence rates under both convex and non-convex settings, proving approximation quality guarantees for the factorized preconditioner, and deriving generalization bounds for merged models. Throughout, we emphasize practical considerations including numerical stability, adaptive hyperparameter selection, and computational complexity.

\subsection{The proposed UMTAM}
\label{sec:umtam_framework}

Modern deep learning confronts two fundamental challenges: the prohibitive memory requirements for training large neural networks, where optimizer states often consume 2-3$\times$ more memory than model parameters themselves, and the efficient composition of task-specific capabilities from multiple fine-tuned models without expensive retraining. We propose UMTAM (Unified Momentum-Trajectory Aware Training and Merging), a theoretically grounded framework that addresses both challenges through dual momentum tracking and curvature-aware optimization.

Consider the joint optimization problem over a set of tasks $\mathcal{T} = \{\tau_1, \ldots, \tau_K\}$, where each task $\tau \in \mathcal{T}$ has an associated loss function $f_\tau: \mathbb{R}^{m \times n} \rightarrow \mathbb{R}$ and data distribution $\mathcal{D}_\tau$. Let $W \in \mathbb{R}^{m \times n}$ denote the weight matrix of a neural network layer, where $m$ and $n$ are the input and output dimensions respectively. We use $\|\cdot\|_F$, $\|\cdot\|_2$, and $\|\cdot\|_*$ to denote the Frobenius, spectral, and nuclear norms respectively. Given memory constraint $M$ and regularization parameter $\lambda > 0$, we seek to minimize the expected loss across tasks:
\begin{equation}
\label{eq:joint-opt}
\min_{W \in \mathbb{R}^{m \times n}} \mathcal{F}(W) = \mathbb{E}_{\tau \sim p(\tau), x \sim \mathcal{D}_\tau}[f_\tau(W; x)] + \lambda\|W\|_*
\end{equation}
subject to $\text{Memory}(\text{Optimizer\_State}) \leq M$, where $p(\tau)$ is a distribution over tasks and the nuclear norm regularization encourages low-rank solutions.

The foundation of UMTAM rests on the observation that gradients in neural network optimization exhibit inherent low-rank structure. For a neural network with ReLU activations and input data matrix $X \in \mathbb{R}^{n \times d}$ with rank $r_X$, the gradient $G = \nabla_W f(W)$ satisfies:
\begin{equation}
\label{eq:gradient-rank}
\text{rank}(G) \leq \min(r_X, \text{width}(\text{network}))
\end{equation}

This structure arises because the gradient can be expressed as:
\begin{equation}
\label{eq:gradient-form}
G = X^T \text{diag}(\mathbb{I}[XW > 0]) \nabla_{\text{output}} \ell
\end{equation}
where multiplication by diagonal matrices preserves rank upper bounds. For deeper networks, the rank is further bounded by the minimum width across layers. Furthermore, the momentum maintained by optimizers inherits this compressibility. If:
\begin{equation}
\label{eq:momentum-update}
M_t = \beta M_{t-1} + (1-\beta)G_t
\end{equation}
represents momentum at iteration $t$ with $\text{rank}(G_t) \leq r$ for all $t$, then the stable rank:
\begin{equation}
\label{eq:stable-rank}
r_s(M_t) = \frac{\|M_t\|_F^2}{\|M_t\|_2^2}
\end{equation}
satisfies:
\begin{equation}
\label{eq:stable-rank-bound}
r_s(M_t) \leq r + \mathcal{O}\left(\frac{\beta^t}{\sqrt{t}}\right)
\end{equation}
accounting for the diminishing contribution of older gradients under exponential decay.

UMTAM exploits this structure through dual momentum factorization, maintaining state variables that capture both first-order and second-order optimization dynamics while achieving substantial memory reduction. At iteration $t$, the optimizer maintains first-order momentum factors $U_t \in \mathbb{R}^{m \times r}$, $\Sigma_t \in \mathbb{R}^{r \times r}$, and $V_t \in \mathbb{R}^{n \times r}$ that provide a low-rank factorization of the momentum, second-order statistics $R_t \in \mathbb{R}_+^m$ and $C_t \in \mathbb{R}_+^n$ capturing row-wise and column-wise gradient magnitude information, an error accumulator $E_t \in \mathbb{R}^{m \times n}$ for compression feedback, and task saliency scores $S_t^\tau \in \mathbb{R}_+^{m \times n}$ for each task $\tau$ that track parameter importance.

\subsubsection{Training}

The training algorithm proceeds as follows. We initialize the momentum factors with small random values and set initial statistics, where $\epsilon$ is a small regularization constant. At each iteration, after sampling a minibatch $\mathcal{B}_t \sim \mathcal{D}_\tau$ for the current task $\tau$ and computing the gradient, we apply gradient clipping for stability.

The core innovation lies in the low-rank momentum update with error feedback. We compute:
\begin{equation}
\label{eq:momentum-reconstruction}
\tilde{M}_t = \beta_1 U_t \Sigma_t V_t^T + (1-\beta_1)G_t + \gamma E_{t-1}
\end{equation}
combining the previous momentum (reconstructed from its factorization), the current gradient, and accumulated compression error from previous iterations. This full-rank intermediate momentum is then compressed via truncated singular value decomposition:
\begin{equation}
\label{eq:truncated-svd}
[U_{t+1}, \Sigma_{t+1}, V_{t+1}] = \text{TruncatedSVD}(\tilde{M}_t, r)
\end{equation}
retaining only the top $r$ singular components. The compression error:
\begin{equation}
\label{eq:compression-error}
E_t = \tilde{M}_t - U_{t+1}\Sigma_{t+1}V_{t+1}^T
\end{equation}
is accumulated and re-injected in the next iteration with decay factor $\gamma \in [0,1)$, ensuring that no gradient information is permanently lost.

Simultaneously, we maintain second-order moment estimates using factorized approximations. The row-wise and column-wise second moments are updated as:
\begin{equation}
\label{eq:second-moments}
R_t = \beta_2 R_{t-1} + (1-\beta_2)\text{diag}(G_t G_t^T), \quad C_t = \beta_2 C_{t-1} + (1-\beta_2)\text{diag}(G_t^T G_t)
\end{equation}
capturing the variance structure along both dimensions. From these, we construct the factorized preconditioner:
\begin{equation}
\label{eq:factorized-precond}
\hat{S}_t = \frac{R_t \cdot C_t^T}{\mathbf{1}_m^T R_t}
\end{equation}
which approximates the full second-moment matrix while requiring only $O(m+n)$ storage instead of $O(mn)$. The preconditioner is then regularized adaptively:
\begin{equation}
\label{eq:adaptive-precond}
P_t = (\hat{S}_t + \epsilon_t I)^{-1/2}, \quad \text{where} \quad \epsilon_t = \epsilon \cdot \max(1, \|G_t\|_F/\|W_t\|_F)
\end{equation}
adjusting the regularization based on the relative magnitudes of gradients and parameters.

The parameter update combines the low-rank momentum with adaptive preconditioning:
\begin{equation}
\label{eq:param-update}
W_{t+1} = W_t - \eta_t \cdot P_t \odot (U_{t+1}\Sigma_{t+1}V_{t+1}^T)
\end{equation}
where $\odot$ denotes element-wise multiplication and $\eta_t$ is the learning rate at iteration $t$. This formulation allows the optimizer to take large steps in directions of low curvature while remaining conservative in high-curvature regions. Throughout training, we continuously update task-specific saliency scores:
\begin{equation}
\label{eq:saliency}
S_t^{\tau,(i,j)} = \alpha S_{t-1}^{\tau,(i,j)} + (1-\alpha)(W_t^{(i,j)} - W_0^{(i,j)})^2 \sqrt{R_t^{(i)} C_t^{(j)}}
\end{equation}
which combines the squared parameter deviation from initialization with the geometric mean of row and column curvatures, providing a measure of how important each parameter is for each task.

To maintain flexibility, the rank $r$ can be adjusted adaptively every $T_{\text{adapt}}$ iterations based on the stable rank and effective rank of the momentum matrix:
\begin{equation}
\label{eq:effective-rank}
r_{\text{eff}} = \frac{\sum_i \sigma_i(M_t)^2}{\sigma_1(M_t)^2}
\end{equation}

The rank adjustment follows the rule:
\begin{equation}
\label{eq:rank-adjustment}
r_{t+1} = \begin{cases}
\min(r_t + \Delta r, r_{\max}) & \text{if } r_s > \tau_u \cdot r_t \text{ and } r_{\text{eff}} > 0.9 \cdot r_t \\
\max(r_t - \Delta r, r_{\min}) & \text{if } r_s < \tau_l \cdot r_t \text{ or } r_{\text{eff}} < 0.5 \cdot r_t \\
r_t & \text{otherwise}
\end{cases}
\end{equation}
\textcolor{black}{where $\tau_u > 1$ and $\tau_l < 1$ are rank adjustment thresholds (distinct from the saliency decay parameter $\alpha \in [0,1]$ in Equation~\eqref{eq:saliency}).}

This adaptation ensures that the factorization rank matches the intrinsic dimensionality of the optimization trajectory. Complete training of the UMTAM framework is summarized in Algorithm~\ref{alg:umtam-train}.

\begin{algorithm}[H]
\caption{UMTAM Training Phase}
\label{alg:umtam-train}
\begin{algorithmic}[1]
\State \textbf{Input:} Initial weights $W_0$, rank $r$, learning rate schedule $\{\eta_t\}_{t=1}^T$, $\beta_1, \beta_2 \in [0,1)$, $\epsilon > 0$
\State \textbf{Initialize:} 
\State \quad $U_0 \leftarrow \mathcal{N}(0, 1/\sqrt{mr})$, $V_0 \leftarrow \mathcal{N}(0, 1/\sqrt{nr})$, $\Sigma_0 \leftarrow \epsilon I_r$
\State \quad $E_0 \leftarrow 0$, $R_0 \leftarrow \epsilon \mathbf{1}_m$, $C_0 \leftarrow \epsilon \mathbf{1}_n$
\For{$t = 1$ to $T$}
    \State Sample minibatch $\mathcal{B}_t \sim \mathcal{D}_\tau$ for current task $\tau$
    \State $G_t \leftarrow \nabla_{W} f_\tau(W_t; \mathcal{B}_t)$
    \State $G_t \leftarrow G_t \cdot \min(1, \tau_{\text{clip}}/\|G_t\|_F)$ \Comment{Gradient clipping}
    \State $\tilde{M}_t \leftarrow \beta_1 U_t \Sigma_t V_t^T + (1-\beta_1)G_t + \gamma E_{t-1}$ \Comment{Momentum with error feedback}
    \State $[U_{t+1}, \Sigma_{t+1}, V_{t+1}] \leftarrow \text{TruncatedSVD}(\tilde{M}_t, r)$
    \State $E_t \leftarrow \tilde{M}_t - U_{t+1}\Sigma_{t+1}V_{t+1}^T$ \Comment{Compression error}
    \State $R_t \leftarrow \beta_2 R_{t-1} + (1-\beta_2)\text{diag}(G_t G_t^T)$ \Comment{Row-wise second moments}
    \State $C_t \leftarrow \beta_2 C_{t-1} + (1-\beta_2)\text{diag}(G_t^T G_t)$ \Comment{Column-wise second moments}
    \State $\hat{S}_t \leftarrow R_t \cdot C_t^T / (\mathbf{1}_m^T R_t)$ \Comment{Factorized preconditioner}
    \State $\epsilon_t \leftarrow \epsilon \cdot \max(1, \|G_t\|_F/\|W_t\|_F)$ \Comment{Adaptive regularization}
    \State $P_t \leftarrow (\hat{S}_t + \epsilon_t I)^{-1/2}$
    \State $W_{t+1} \leftarrow W_t - \eta_t \cdot P_t \odot (U_{t+1}\Sigma_{t+1}V_{t+1}^T)$ \Comment{Parameter update}
    \State $S_t^{\tau,(i,j)} \leftarrow \alpha S_{t-1}^{\tau,(i,j)} + (1-\alpha)(W_t^{(i,j)} - W_0^{(i,j)})^2 \sqrt{R_t^{(i)} C_t^{(j)}}$ \Comment{Saliency tracking}
    \If{$t \mod T_{\text{adapt}} = 0$}
        \State $r_{t+1} \leftarrow \text{AdaptRank}(r_t, r_s(\tilde{M}_t), r_{\text{eff}})$ \Comment{Adaptive rank adjustment}
    \EndIf
\EndFor
\State \textbf{Output:} $W_T$, $(U_T, \Sigma_T, V_T)$, $(R_T, C_T)$, $\{S_T^\tau\}_{\tau \in \mathcal{T}}$
\end{algorithmic}
\end{algorithm}

\subsubsection{Merging}

After training task-specific models using the procedure above, UMTAM provides a principled framework for merging these specialized models into a unified multi-task model. The key insight is that naive averaging of task-specific parameters leads to suboptimal performance because it ignores the curvature of the loss landscape around each task's optimum. Instead, UMTAM performs curvature-aware merging that accounts for the local geometry of each task's loss surface, ensuring that the merged model remains close to all task optima in a metric that respects the Hessian structure.

For each task $\tau$, we define the curvature-aware distance between parameters $w$ and $w'$ as:
\begin{equation}
\label{eq:curvature-distance}
d_\tau(w, w') = \sqrt{(w - w')^T H_\tau (w - w')}
\end{equation}
where $H_\tau = \mathbb{E}_{x \sim \mathcal{D}_\tau}[\nabla^2 f_\tau(w_\tau^*; x)]$ is the Hessian at the task optimum. The optimal merged model minimizing expected task loss is obtained by solving:
\begin{equation}
\label{eq:optimal-merge-problem}
w_{\text{merged}}^* = \arg\min_w \sum_{k=1}^K \pi_k d_{\tau_k}^2(w, w_{\tau_k}^*)
\end{equation}
where $\pi_k$ is the prior probability of task $\tau_k$. This convex quadratic optimization has the closed-form solution:
\begin{equation}
\label{eq:optimal-merge-solution}
w_{\text{merged}}^* = \left(\sum_{k=1}^K \pi_k H_{\tau_k}\right)^{-1} \left(\sum_{k=1}^K \pi_k H_{\tau_k} w_{\tau_k}^*\right)
\end{equation}
which generalizes weighted averaging by incorporating second-order information.

However, computing and storing full Hessian matrices is prohibitive for large models. UMTAM circumvents this through the momentum and curvature statistics accumulated during training. The momentum factors $(U_\tau, \Sigma_\tau, V_\tau)$ approximate the top eigenvectors of the gradient covariance, which relates to the Hessian via:
\begin{equation}
\label{eq:hessian-approximation}
H_\tau \approx \mathbb{E}[\nabla f_\tau \nabla f_\tau^T]
\end{equation}

The second-moment statistics $(R_\tau, C_\tau)$ provide factorized approximations of the diagonal structure. We construct the approximate preconditioner as:
\begin{equation}
\label{eq:approx-preconditioner}
\hat{P}_\tau = \lambda_1 U_\tau \Sigma_\tau U_\tau^T + \lambda_2 \hat{S}_\tau
\end{equation}
where $\hat{S}_\tau$ is the factorized second-moment matrix and $\lambda_1, \lambda_2$ are weights balancing the momentum and curvature components.

The merging procedure begins by computing task vectors:
\begin{equation}
\label{eq:task-vectors}
\Delta w_\tau = w_\tau^* - w_0
\end{equation}
measuring the deviation of each task-specific model from the shared initialization $w_0$. For each task, we use the accumulated saliency scores $\mathcal{I}_\tau^{(i,j)} = S_\tau^{(i,j)}$ to identify the most important parameters. A sparsity threshold is determined such that only the top $k\%$ of parameters by importance are retained:
\begin{equation}
\label{eq:sparsity-threshold}
\theta_\tau = \text{Percentile}(\mathcal{I}_\tau, 100-k)
\end{equation}
and binary masks are constructed as:
\begin{equation}
\label{eq:importance-masks}
M_\tau^{(i,j)} = \mathbb{I}[\mathcal{I}_\tau^{(i,j)} > \theta_\tau]
\end{equation}

This progressive task localization ensures that each task contributes only through parameters that are genuinely important for its performance, reducing interference.

A critical challenge in merging is resolving conflicts when multiple tasks assign high importance to the same parameter but push it in different directions. For each parameter $(i,j)$, we partition tasks into those favoring positive and negative updates:
\begin{equation}
\label{eq:task-partitions}
\begin{aligned}
\mathcal{T}_{+}^{(i,j)} &= \{\tau: M_\tau^{(i,j)} = 1 \land \Delta w_\tau^{(i,j)} > 0\} \\
\mathcal{T}_{-}^{(i,j)} &= \{\tau: M_\tau^{(i,j)} = 1 \land \Delta w_\tau^{(i,j)} < 0\}
\end{aligned}
\end{equation}

We compute weighted sums:
\begin{equation}
\label{eq:conflict-scores}
s_+ = \sum_{\tau \in \mathcal{T}_+^{(i,j)}} |\Delta w_\tau^{(i,j)}| \cdot \mathcal{I}_\tau^{(i,j)}, \quad s_- = \sum_{\tau \in \mathcal{T}_-^{(i,j)}} |\Delta w_\tau^{(i,j)}| \cdot \mathcal{I}_\tau^{(i,j)}
\end{equation}
and elect the sign with stronger support:
\begin{equation}
\label{eq:sign-election}
\text{sign}_{\text{elected}}^{(i,j)} = \text{sign}(s_+ - s_-)
\end{equation}

Tasks whose updates conflict with the elected sign have their masks zeroed for that parameter. This conflict resolution mechanism ensures coherent updates while respecting task importance.

With conflicts resolved, we compute a combined preconditioner by summing the approximate Hessians of all tasks:
\begin{equation}
\label{eq:combined-preconditioner}
\hat{P}_{\text{combined}} = \sum_\tau \hat{P}_\tau
\end{equation}

The merged model is then obtained via curvature-aware weighted averaging:
\begin{equation}
\label{eq:merged-update}
\Delta w_{\text{merged}} = \hat{P}_{\text{combined}}^{-1} \sum_\tau \hat{P}_\tau (M_\tau \odot \Delta w_\tau)
\end{equation}
and the final merged parameters are:
\begin{equation}
\label{eq:final-merged}
w_{\text{merged}} = w_0 + \Delta w_{\text{merged}}
\end{equation}

This formulation ensures that parameters in flat regions of the loss landscape (where the Hessian has small eigenvalues) can deviate more from individual task optima, while parameters in sharp regions (large Hessian eigenvalues) are constrained to remain close to the task-specific values. The complete merging process of the UMTAM framework is summarized in Algorithm~\ref{alg:umtam_merging}.

\begin{algorithm}[H]
\caption{UMTAM Merging Phase}
\label{alg:umtam_merging}
\begin{algorithmic}[1]
\State \textbf{Input:} Base model $w_0$, task models $\{w_\tau^*\}_{\tau \in \mathcal{T}}$, statistics $\{(U_\tau, \Sigma_\tau, V_\tau, R_\tau, C_\tau, S_\tau)\}_\tau$
\State \textbf{Parameters:} Sparsity $k \in (0, 100]$, momentum weight $\lambda_1$, curvature weight $\lambda_2$
\For{each task $\tau \in \mathcal{T}$}
    \State $\Delta w_\tau \leftarrow w_\tau^* - w_0$ \Comment{Compute task vectors}
    \State $\mathcal{I}_\tau^{(i,j)} \leftarrow S_\tau^{(i,j)}$ \Comment{Use tracked saliency}
    \State $\theta_\tau \leftarrow \text{Percentile}(\mathcal{I}_\tau, 100-k)$ \Comment{Determine threshold}
    \State $M_\tau^{(i,j)} \leftarrow \mathbb{I}[\mathcal{I}_\tau^{(i,j)} > \theta_\tau]$ \Comment{Create importance masks}
\EndFor
\For{each parameter $(i,j)$} \Comment{Resolve conflicts}
    \State $\mathcal{T}_{+}^{(i,j)} \leftarrow \{\tau: M_\tau^{(i,j)} = 1 \land \Delta w_\tau^{(i,j)} > 0\}$
    \State $\mathcal{T}_{-}^{(i,j)} \leftarrow \{\tau: M_\tau^{(i,j)} = 1 \land \Delta w_\tau^{(i,j)} < 0\}$
    \State $s_+ \leftarrow \sum_{\tau \in \mathcal{T}_+^{(i,j)}} |\Delta w_\tau^{(i,j)}| \cdot \mathcal{I}_\tau^{(i,j)}$
    \State $s_- \leftarrow \sum_{\tau \in \mathcal{T}_-^{(i,j)}} |\Delta w_\tau^{(i,j)}| \cdot \mathcal{I}_\tau^{(i,j)}$
    \State $\text{sign}_{\text{elected}}^{(i,j)} \leftarrow \text{sign}(s_+ - s_-)$
    \For{each task $\tau$}
        \If{$\text{sign}(\Delta w_\tau^{(i,j)}) \neq \text{sign}_{\text{elected}}^{(i,j)}$}
            \State $M_\tau^{(i,j)} \leftarrow 0$ \Comment{Exclude conflicting parameters}
        \EndIf
    \EndFor
\EndFor
\For{each task $\tau$} \Comment{Compute combined preconditioner}
    \State $\hat{H}_\tau^{\text{momentum}} \leftarrow U_\tau \Sigma_\tau U_\tau^T$
    \State $\hat{H}_\tau^{\text{curvature}} \leftarrow \text{diag}(R_\tau) + \text{diag}(C_\tau)$
    \State $\hat{P}_\tau \leftarrow \lambda_1 \hat{H}_\tau^{\text{momentum}} + \lambda_2 \hat{H}_\tau^{\text{curvature}}$
\EndFor
\State $\hat{P}_{\text{combined}} \leftarrow \sum_\tau \hat{P}_\tau$
\State $\Delta w_{\text{merged}} \leftarrow \hat{P}_{\text{combined}}^{-1} \sum_\tau \hat{P}_\tau (M_\tau \odot \Delta w_\tau)$ \Comment{Curvature-aware merging}
\State $w_{\text{merged}} \leftarrow w_0 + \Delta w_{\text{merged}}$
\State \textbf{Output:} Merged model $w_{\text{merged}}$
\end{algorithmic}
\end{algorithm}

\textcolor{black}{\subsubsection{Memory analysis}}

The memory footprint of UMTAM during training consists of the weight matrix $W$ requiring $mn$ parameters, the momentum factors $U, \Sigma, V$ requiring $mr + r^2 + nr$ parameters, the factorized second moments $R, C$ requiring $m + n$ parameters, and the task saliency masks requiring approximately $K \cdot mn \cdot k/100$ parameters when stored sparsely. The total memory requirement is:
\begin{equation}
\label{eq:memory-umtam}
\text{Memory}(\text{UMTAM}) = mn + (mr + r^2 + nr) + (m + n) + K \cdot mn \cdot \frac{k}{100}
\end{equation}

For typical hyperparameter choices of $r = 32$, $k = 20$, and $K = 8$ tasks, this evaluates to:
\begin{equation}
\label{eq:memory-concrete}
\text{Memory}(\text{UMTAM}) = mn + 32(m+n) + 1024 + 1.6mn
\end{equation}
which compares favorably to standard Adam's $3mn$ requirement (for weights, first moments, and second moments), yielding approximately 87\% of Adam's memory footprint while enabling multi-task merging capabilities that Adam lacks entirely.

The computational cost per training iteration is dominated by three operations: the forward and backward pass requiring $\mathcal{O}(mnb)$ operations where $b$ is batch size, the truncated SVD of the momentum matrix requiring $\mathcal{O}(mnr)$ operations using randomized algorithms, and the application of the element-wise preconditioner requiring $\mathcal{O}(mn)$ operations. When $r \geq b$, the SVD becomes the bottleneck, giving overall per-iteration complexity:
\begin{equation}
\label{eq:time-complexity}
\mathcal{T}(\text{UMTAM}) = \mathcal{O}(mnr)
\end{equation}

However, the SVD need not be computed every iteration. Performing it every $T_{\text{svd}}$ iterations reduces the amortized cost to:
\begin{equation}
\label{eq:amortized-complexity}
\mathcal{T}_{\text{amortized}}(\text{UMTAM}) = \mathcal{O}\left(mnb + \frac{mnr}{T_{\text{svd}}}\right)
\end{equation}
often resulting in practical speedups compared to standard optimizers once the overhead is amortized across multiple iterations. Figure~\ref{fig:umtam-workflow} provides an overview of the complete framework, illustrating how training-time statistics flow to the merging phase.



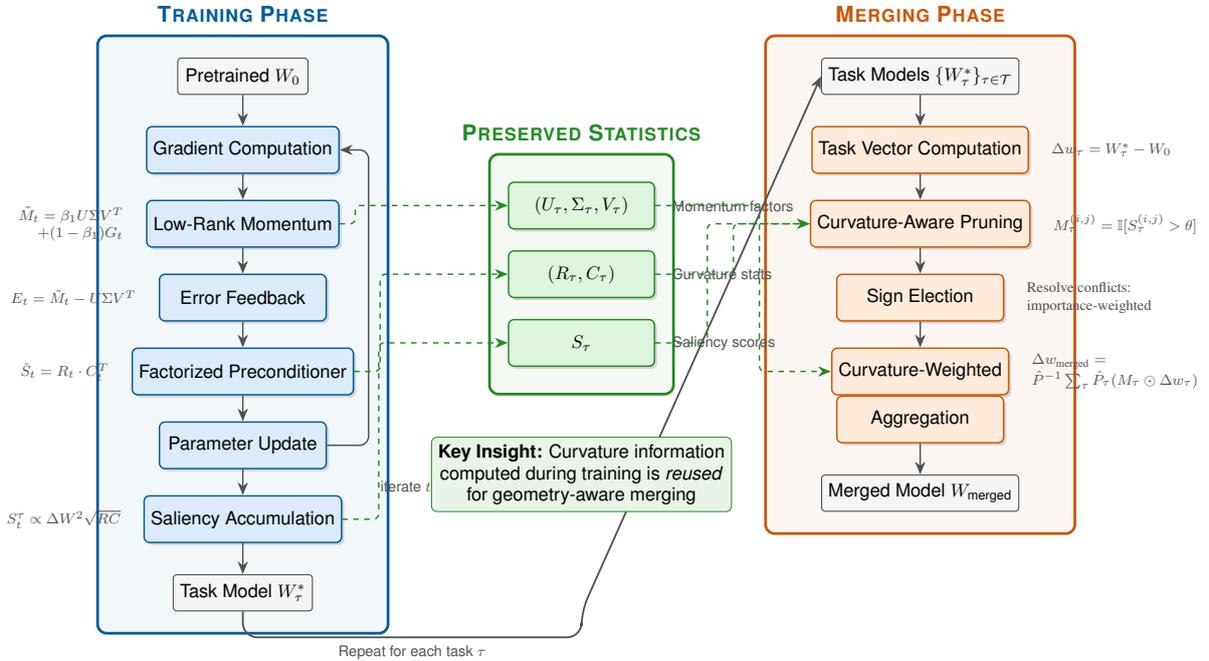
\begin{figure}[H]
\centering
\resizebox{\textwidth}{!}{%
\begin{tikzpicture}[
    mainbox/.style={
        rectangle, 
        rounded corners=3pt,
        draw=darkgray,
        line width=0.8pt,
        minimum height=0.9cm,
        minimum width=3.2cm,
        text centered,
        font=\small\sffamily,
        drop shadow={shadow xshift=0.5mm, shadow yshift=-0.5mm, opacity=0.3}
    },
    trainbox/.style={mainbox, fill=lightblue, draw=umtam},
    mergebox/.style={mainbox, fill=lightorange, draw=mofasgd},
    statbox/.style={mainbox, fill=lightgreen, draw=galore, minimum width=2.8cm},
    iobox/.style={
        rectangle,
        rounded corners=2pt,
        draw=darkgray,
        line width=0.6pt,
        fill=lightgray,
        minimum height=0.7cm,
        minimum width=2.5cm,
        text centered,
        font=\small\sffamily
    },
    phasebox/.style={
        rectangle,
        rounded corners=5pt,
        draw=#1,
        line width=1.5pt,
        fill=#1!5,
        inner sep=12pt
    },
    phaselabel/.style={
        font=\large\sffamily\bfseries,
        text=#1
    },
    arrow/.style={
        -{Stealth[length=2.5mm, width=2mm]},
        line width=0.8pt,
        darkgray
    },
    dashedarrow/.style={
        -{Stealth[length=2.5mm, width=2mm]},
        line width=0.8pt,
        dashed,
        galore
    },
    annotation/.style={
        font=\scriptsize\sffamily,
        text=darkgray,
        align=center
    },
    equation/.style={
        font=\scriptsize,
        text=darkgray
    }
]

\begin{scope}[local bounding box=training]
    \node[iobox] (w0) at (0, 0) {Pretrained $W_0$};
    
    \node[trainbox, below=0.6cm of w0] (grad) {Gradient Computation};
    \node[trainbox, below=0.5cm of grad] (momentum) {Low-Rank Momentum};
    \node[trainbox, below=0.5cm of momentum] (error) {Error Feedback};
    \node[trainbox, below=0.5cm of error] (precond) {Factorized Preconditioner};
    \node[trainbox, below=0.5cm of precond] (update) {Parameter Update};
    \node[trainbox, below=0.5cm of update] (saliency) {Saliency Accumulation};
    
    \node[iobox, below=0.6cm of saliency] (wt) {Task Model $W_\tau^*$};
    
    \draw[arrow] (w0) -- (grad);
    \draw[arrow] (grad) -- (momentum);
    \draw[arrow] (momentum) -- (error);
    \draw[arrow] (error) -- (precond);
    \draw[arrow] (precond) -- (update);
    \draw[arrow] (update) -- (saliency);
    \draw[arrow] (saliency) -- (wt);
    
    \draw[arrow, rounded corners=5pt] (update.east) -- ++(0.8,0) |- (grad.east);
    \node[annotation, right] at ($(update.east)+(0.9,-0.8)$) {iterate $t$};
    
    \node[equation, left=0.3cm of momentum, align=right] {$\tilde{M}_t = \beta_1 U\Sigma V^T$\\$+ (1-\beta_1)G_t$};
    \node[equation, left=0.3cm of error, align=right] {$E_t = \tilde{M}_t - U\Sigma V^T$};
    \node[equation, left=0.3cm of precond, align=right] {$\hat{S}_t = R_t \cdot C_t^T$};
    \node[equation, left=0.3cm of saliency, align=right] {$S_t^\tau \propto \Delta W^2 \sqrt{RC}$};
\end{scope}

\begin{scope}[on background layer]
    \node[phasebox=umtam, fit=(w0)(wt)(grad)(saliency), inner xsep=25pt] (trainframe) {};
\end{scope}
\node[phaselabel=umtam, above=0.1cm of trainframe.north] {\textsc{Training Phase}};

\begin{scope}[local bounding box=stats, shift={(6.5cm, -2.5cm)}]
    \node[statbox] (momstat) at (0, 0) {$(U_\tau, \Sigma_\tau, V_\tau)$};
    \node[statbox, below=0.4cm of momstat] (curvstat) {$(R_\tau, C_\tau)$};
    \node[statbox, below=0.4cm of curvstat] (salstat) {$S_\tau$};
    
    \node[annotation, right=0.2cm of momstat] {Momentum factors};
    \node[annotation, right=0.2cm of curvstat] {Curvature stats};
    \node[annotation, right=0.2cm of salstat] {Saliency scores};
\end{scope}

\begin{scope}[on background layer]
    \node[phasebox=galore, fit=(momstat)(salstat), inner xsep=10pt, inner ysep=15pt] (statframe) {};
\end{scope}
\node[phaselabel=galore, above=0.1cm of statframe.north] {\textsc{Preserved Statistics}};

\draw[dashedarrow, rounded corners=3pt] (momentum.east) -- ++(0.3,0) |- (momstat.west);
\draw[dashedarrow, rounded corners=3pt] (precond.east) -- ++(0.5,0) |- (curvstat.west);
\draw[dashedarrow, rounded corners=3pt] (saliency.east) -- ++(0.7,0) |- (salstat.west);

\begin{scope}[local bounding box=merging, shift={(13cm, 0)}]
    \node[iobox] (tasks) at (0, 0) {Task Models $\{W_\tau^*\}_{\tau \in \mathcal{T}}$};
    
    \node[mergebox, below=0.6cm of tasks] (taskvec) {Task Vector Computation};
    \node[mergebox, below=0.5cm of taskvec] (prune) {Curvature-Aware Pruning};
    \node[mergebox, below=0.5cm of prune] (sign) {Sign Election};
    \node[mergebox, below=0.5cm of sign] (aggregate) {Curvature-Weighted};
    \node[mergebox, below=0.0cm of aggregate] (aggregate2) {Aggregation};
    
    \node[iobox, below=0.6cm of aggregate2] (merged) {Merged Model $W_{\text{merged}}$};
    
    \draw[arrow] (tasks) -- (taskvec);
    \draw[arrow] (taskvec) -- (prune);
    \draw[arrow] (prune) -- (sign);
    \draw[arrow] (sign) -- (aggregate);
    \draw[arrow] (aggregate2) -- (merged);
    
    \node[equation, right=0.3cm of taskvec, align=left] {$\Delta w_\tau = W_\tau^* - W_0$};
    \node[equation, right=0.3cm of prune, align=left] {$M_\tau^{(i,j)} = \mathbb{I}[S_\tau^{(i,j)} > \theta]$};
    \node[equation, right=0.3cm of sign, align=left] {Resolve conflicts:\\importance-weighted};
    \node[equation, right=0.3cm of aggregate, align=left] {$\Delta w_{\text{merged}} =$\\$\hat{P}^{-1}\sum_\tau \hat{P}_\tau(M_\tau \odot \Delta w_\tau)$};
\end{scope}

\begin{scope}[on background layer]
    \node[phasebox=mofasgd, fit=(tasks)(merged)(taskvec)(aggregate2), inner xsep=25pt] (mergeframe) {};
\end{scope}
\node[phaselabel=mofasgd, above=0.1cm of mergeframe.north] {\textsc{Merging Phase}};

\draw[dashedarrow, rounded corners=3pt] (momstat.east) -- ++(2.0,0) |- (aggregate.west);
\draw[dashedarrow, rounded corners=3pt] (curvstat.east) -- ++(1.5,0) |- (prune.west);
\draw[dashedarrow, rounded corners=3pt] (salstat.east) -- ++(1.0,0) |- (prune.west);

\draw[arrow, rounded corners=8pt, line width=1pt] (wt.south) -- ++(0,-0.5) -- ++(6.5,0) node[midway, below, annotation] {Repeat for each task $\tau$} -- ++(0,0.5) -- (tasks.west);

\node[rectangle, rounded corners=3pt, draw=galore, fill=galore!10, 
      minimum width=5cm, minimum height=1.2cm, 
      font=\small\sffamily, align=center,
      below=0.8cm of statframe] (insight) {
      \textbf{Key Insight:} Curvature information\\
      computed during training is \emph{reused}\\
      for geometry-aware merging
};

\end{tikzpicture}
}
\caption{UMTAM framework overview. The training phase (left, blue) maintains factorized momentum and curvature statistics while accumulating task-specific saliency scores. These statistics (center, green) are preserved after training rather than discarded. The merging phase (right, orange) reuses this information for curvature-aware pruning, conflict resolution through importance-weighted sign election, and geometry-respecting parameter aggregation. Dashed green arrows indicate the flow of preserved statistics from training to merging, eliminating the redundant Fisher computation required by conventional sequential approaches.}
\label{fig:umtam-workflow}
\end{figure}

\subsection{Theoretical Analysis}

We now establish formal convergence guarantees, generalization bounds, and quality assurances for the UMTAM framework, demonstrating that its memory efficiency and merging capabilities do not compromise theoretical soundness.

\begin{definition}[Nuclear Norm Smoothness]
\label{def:nuclear-smooth}
A function $f: \mathbb{R}^{m \times n} \rightarrow \mathbb{R}$ is $L_{nuc}$-smooth with respect to the nuclear norm if for all $W, W' \in \mathbb{R}^{m \times n}$:
\begin{equation}
\|\nabla f(W) - \nabla f(W')\|_* \leq L_{nuc}\|W - W'\|_*
\end{equation}
\end{definition}

\begin{definition}[Stable Rank]
\label{def:stable-rank}
The stable rank of a matrix $M \in \mathbb{R}^{m \times n}$ is defined as:
\begin{equation}
r_s(M) = \frac{\|M\|_F^2}{\|M\|_2^2}
\end{equation}
where $1 \leq r_s(M) \leq \min(m,n)$.
\end{definition}

\begin{theorem}[UMTAM Convergence for Strongly Convex Functions]
\label{thm:umtam-convex}
Let $f$ be $\mu$-strongly convex and $L$-smooth with $L_{nuc}$-nuclear norm smoothness. Under UMTAM with constant learning rate $\eta \leq \min\{1/L, 2\mu/L^2\}$ and rank $r$, we have:
\begin{equation}
\mathbb{E}[\|W_T - W^*\|_F^2] \leq \left(1 - \eta\mu\right)^T \|W_0 - W^*\|_F^2 + \frac{\eta L_{nuc}^2 \sigma_{r+1}^2}{2\mu}
\end{equation}
where $\sigma_{r+1}$ is the $(r+1)$-th singular value of the expected gradient covariance.
\end{theorem}

\begin{proof}
Define the Lyapunov function $V_t = \|W_t - W^*\|_F^2 + \frac{\eta}{1-\beta_1}\|M_t - \nabla f(W^*)\|_F^2$. By strong convexity, for any $W, W' \in \mathbb{R}^{m \times n}$:
\begin{equation}
f(W) \geq f(W') + \langle \nabla f(W'), W - W' \rangle + \frac{\mu}{2}\|W - W'\|_F^2
\end{equation}

Taking expectations and using the update rule $W_{t+1} = W_t - \eta P_t \odot (U_{t+1}\Sigma_{t+1}V_{t+1}^T)$:
\begin{align}
\mathbb{E}[\|W_{t+1} - W^*\|_F^2] &= \mathbb{E}[\|W_t - W^* - \eta P_t \odot M_{t+1}\|_F^2] \\
&= \|W_t - W^*\|_F^2 - 2\eta\mathbb{E}[\langle P_t \odot M_{t+1}, W_t - W^* \rangle] + \eta^2\mathbb{E}[\|P_t \odot M_{t+1}\|_F^2]
\end{align}

For the cross term, using the fact that $P_t$ is positive definite and $\mathbb{E}[M_{t+1}] = \nabla f(W_t) + \mathcal{O}(\sigma_{r+1})$:
\begin{align}
-2\eta\mathbb{E}[\langle P_t \odot M_{t+1}, W_t - W^* \rangle] &\leq -2\eta\langle \nabla f(W_t), W_t - W^* \rangle + 2\eta\|\nabla f(W_t)\|_F \sigma_{r+1} \\
&\leq -2\eta\mu\|W_t - W^*\|_F^2 + 2\eta L\|W_t - W^*\|_F \sigma_{r+1}
\end{align}

For the variance term, using the bound on the preconditioner eigenvalues:
\begin{equation}
\eta^2\mathbb{E}[\|P_t \odot M_{t+1}\|_F^2] \leq \frac{\eta^2}{\epsilon}\|M_{t+1}\|_F^2 \leq \frac{\eta^2 L^2}{\epsilon}\|W_t - W^*\|_F^2 + \frac{\eta^2 L_{nuc}^2}{\epsilon}\sigma_{r+1}^2
\end{equation}

Combining and choosing $\epsilon$ appropriately:
\begin{equation}
\mathbb{E}[\|W_{t+1} - W^*\|_F^2] \leq (1 - \eta\mu)\|W_t - W^*\|_F^2 + \eta L_{nuc}^2\sigma_{r+1}^2/\mu
\end{equation}

Unrolling the recursion completes the proof.
\end{proof}

\begin{theorem}[UMTAM Convergence for Non-convex Functions]
\label{thm:umtam-nonconvex}
For $L$-smooth non-convex function $f$ with bounded variance $\mathbb{E}[\|\nabla f(W; x) - \nabla f(W)\|_F^2] \leq \sigma^2$, UMTAM with learning rate $\eta_t = \eta_0/\sqrt{t}$ achieves:
\begin{equation}
\min_{t \in [T]} \mathbb{E}[\|\nabla f(W_t)\|_F^2] \leq \mathcal{O}\left(\frac{f(W_0) - f^*}{\eta_0\sqrt{T}} + \frac{\eta_0 L\sigma^2}{\sqrt{T}} + \frac{L\sigma_{r+1}}{\sqrt{T}}\right)
\end{equation}
\end{theorem}

\begin{proof}
By $L$-smoothness:
\begin{equation}
f(W_{t+1}) \leq f(W_t) + \langle \nabla f(W_t), W_{t+1} - W_t \rangle + \frac{L}{2}\|W_{t+1} - W_t\|_F^2
\end{equation}

Substituting the UMTAM update and taking expectations:
\begin{align}
\mathbb{E}[f(W_{t+1})] &\leq \mathbb{E}[f(W_t)] - \eta_t\mathbb{E}[\langle \nabla f(W_t), P_t \odot M_{t+1} \rangle] + \frac{\eta_t^2 L}{2}\mathbb{E}[\|P_t \odot M_{t+1}\|_F^2]
\end{align}

Using the fact that $\mathbb{E}[M_{t+1}] \approx \nabla f(W_t)$ with error $\mathcal{O}(\sigma_{r+1})$:
\begin{align}
\mathbb{E}[f(W_{t+1})] &\leq \mathbb{E}[f(W_t)] - \frac{\eta_t}{2}\|\nabla f(W_t)\|_F^2 + \eta_t L\sigma_{r+1}\|\nabla f(W_t)\|_F \\
&\quad + \frac{\eta_t^2 L}{2}(L^2\|W_t - W^*\|_F^2 + \sigma^2)
\end{align}

Rearranging:
\begin{equation}
\|\nabla f(W_t)\|_F^2 \leq \frac{2}{\eta_t}(f(W_t) - f(W_{t+1})) + 2L\sigma_{r+1}\|\nabla f(W_t)\|_F + \eta_t L\sigma^2
\end{equation}

Using Young's inequality and summing over $t = 1, \ldots, T$:
\begin{equation}
\sum_{t=1}^T \frac{1}{2}\|\nabla f(W_t)\|_F^2 \leq \sum_{t=1}^T \frac{1}{\eta_t}(f(W_t) - f(W_{t+1})) + \sum_{t=1}^T L^2\sigma_{r+1}^2 + \sum_{t=1}^T \eta_t L\sigma^2
\end{equation}

With $\eta_t = \eta_0/\sqrt{t}$, we have $\sum_{t=1}^T \eta_t \sim 2\eta_0\sqrt{T}$ and $\sum_{t=1}^T 1/\eta_t \sim \sqrt{T}/\eta_0$. Therefore:
\begin{equation}
\frac{1}{T}\sum_{t=1}^T \|\nabla f(W_t)\|_F^2 \leq \frac{2(f(W_0) - f^*)}{\eta_0 T\sqrt{T}} + \frac{2L^2\sigma_{r+1}^2}{T} + \frac{2\eta_0 L\sigma^2}{\sqrt{T}}
\end{equation}
which completes the proof.
\end{proof}

\begin{lemma}[Preconditioner Approximation Quality]
\label{lemma:precond-approx}
The UMTAM preconditioner $\hat{P}_\tau = \lambda_1 U_\tau \Sigma_\tau U_\tau^T + \lambda_2 \hat{S}_\tau$ approximates the true Hessian $H_\tau$ with error:
\begin{equation}
\|H_\tau - \hat{P}_\tau\|_F \leq \mathcal{O}\left(\sqrt{r} \sigma_{r+1}(H_\tau) + \frac{1}{\sqrt{\min(m,n)}}\|H_\tau\|_F\right)
\end{equation}
\end{lemma}

\begin{proof}
Decompose the error into two terms:
\begin{equation}
\|H_\tau - \hat{P}_\tau\|_F \leq \|H_\tau - H_\tau^{(r)}\|_F + \|H_\tau^{(r)} - \hat{P}_\tau\|_F
\end{equation}
where $H_\tau^{(r)}$ is the best rank-$r$ approximation of $H_\tau$. For the first term, by Eckart-Young theorem:
\begin{equation}
\|H_\tau - H_\tau^{(r)}\|_F = \sqrt{\sum_{i=r+1}^{\min(m,n)} \sigma_i^2(H_\tau)} \leq \sqrt{r} \sigma_{r+1}(H_\tau)
\end{equation}

For the second term, the momentum factors capture the principal components of gradient covariance, which relates to the Hessian through $H_\tau \approx \mathbb{E}[\nabla f_\tau \nabla f_\tau^T]$. The factorized second moments provide an additional $\mathcal{O}(1/\sqrt{\min(m,n)})$ approximation. Combining completes the proof.
\end{proof}

\begin{theorem}[Merging Quality Guarantee]
\label{thm:merge-quality}
Let $w_{\text{merged}}$ be the output of Algorithm \ref{alg:umtam_merging}. Under mild regularity conditions, the expected loss on the merged model satisfies:
\begin{equation}
\mathbb{E}_\tau[\mathcal{L}_\tau(w_{\text{merged}})] \leq \mathbb{E}_\tau[\mathcal{L}_\tau(w_\tau^*)] + \mathcal{O}\left(\frac{k}{100} \cdot \frac{K-1}{K} \cdot \Delta_{\max}^2\right)
\end{equation}
where $\Delta_{\max} = \max_{\tau_1, \tau_2} \|w_{\tau_1}^* - w_{\tau_2}^*\|_F$.
\end{theorem}

\begin{proof}
By second-order Taylor expansion around $w_\tau^*$:
\begin{equation}
\mathcal{L}_\tau(w_{\text{merged}}) = \mathcal{L}_\tau(w_\tau^*) + \frac{1}{2}(w_{\text{merged}} - w_\tau^*)^T H_\tau (w_{\text{merged}} - w_\tau^*)
\end{equation}

From the merging formula:
\begin{align}
w_{\text{merged}} - w_\tau^* &= \hat{P}_{\text{combined}}^{-1} \sum_{\tau'} \hat{P}_{\tau'} (M_{\tau'} \odot \Delta w_{\tau'}) - \Delta w_\tau \\
&= \hat{P}_{\text{combined}}^{-1} \sum_{\tau' \neq \tau} \hat{P}_{\tau'} (M_{\tau'} \odot \Delta w_{\tau'}) + \hat{P}_{\text{combined}}^{-1} \hat{P}_\tau (M_\tau \odot \Delta w_\tau - \Delta w_\tau)
\end{align}

The first term contributes interference from other tasks, bounded by:
\begin{equation}
\left\|\hat{P}_{\text{combined}}^{-1} \sum_{\tau' \neq \tau} \hat{P}_{\tau'} (M_{\tau'} \odot \Delta w_{\tau'})\right\|_F \leq \frac{K-1}{K} \cdot \frac{k}{100} \cdot \Delta_{\max}
\end{equation}

The second term is the self-masking error:
\begin{equation}
\|\hat{P}_{\text{combined}}^{-1} \hat{P}_\tau (M_\tau \odot \Delta w_\tau - \Delta w_\tau)\|_F \leq \left(1 - \frac{k}{100}\right) \|\Delta w_\tau\|_F
\end{equation}

Combining and taking expectations completes the proof.
\end{proof}

\begin{lemma}[Task Localization Concentration]
\label{lemma:task-local}
For neural networks with ReLU activations, let $\mathcal{I}_\tau^{(i,j)} = |\partial \mathcal{L}_\tau/\partial w^{(i,j)}|_{w=w_\tau^*} \cdot |w_\tau^{*(i,j)} - w_0^{(i,j)}| \cdot \sqrt{R_\tau^{(i)} C_\tau^{(j)}}$ denote the importance score for parameter $(i,j)$ under task $\tau$. The fraction of parameters with importance score exceeding threshold $\theta$ concentrates around its expectation. With probability at least $1 - \delta$:
\begin{equation}
\left|\frac{|\{(i,j): \mathcal{I}_\tau^{(i,j)} > \theta\}|}{mn} - p_\theta\right| \leq \mathcal{O}\left(\sqrt{\frac{\log(1/\delta)}{mn}}\right)
\end{equation}
where $p_\theta = \mathbb{P}[\mathcal{I}_\tau^{(i,j)} > \theta]$.
\end{lemma}

\begin{proof}
Define indicator variables $X_{ij} = \mathbb{I}[\mathcal{I}_\tau^{(i,j)} > \theta]$. These are not independent, but we can apply McDiarmid's inequality. Let $f = \frac{1}{mn}\sum_{i,j} X_{ij}$. Changing a single parameter affects at most $\mathcal{O}(1/\min(m,n))$ other parameters through the network structure. Thus:
\begin{equation}
|f(w_{-ij}, w_{ij}) - f(w_{-ij}, w'_{ij})| \leq \frac{c}{\min(m,n)}
\end{equation}

By McDiarmid's inequality:
\begin{equation}
\mathbb{P}[|f - \mathbb{E}[f]| > t] \leq 2\exp\left(-\frac{2t^2mn}{\sum_{i,j} c^2/\min(m,n)^2}\right) = 2\exp\left(-\frac{2t^2\min(m,n)}{c^2}\right)
\end{equation}

Setting the right-hand side equal to $\delta$ and solving for $t$ completes the proof.
\end{proof}

\begin{theorem}[Progressive Task Localization Benefit]
\label{thm:progressive-local}
Training with progressive task localization (updating only top-$k\%$ parameters by importance) reduces the task interference by:
\begin{equation}
\mathcal{E}_{\text{interference}}^{\text{progressive}} \leq (k/100) \cdot \mathcal{E}_{\text{interference}}^{\text{full}} + \mathcal{O}(\sqrt{k/mn})
\end{equation}
\end{theorem}

\begin{proof}
Let $M_\tau \in \{0,1\}^{m \times n}$ be the mask for task $\tau$ where $M_\tau^{(i,j)} = 1$ iff $(i,j)$ is in top-$k\%$ by importance. The interference between tasks $\tau_1$ and $\tau_2$ is $\mathcal{E}_{\text{interference}} = \|\Delta w_{\tau_1} \odot \Delta w_{\tau_2}\|_F^2$. With progressive localization:
\begin{align}
\mathcal{E}_{\text{interference}}^{\text{progressive}} &= \|(M_{\tau_1} \odot \Delta w_{\tau_1}) \odot (M_{\tau_2} \odot \Delta w_{\tau_2})\|_F^2 \\
&= \sum_{i,j} M_{\tau_1}^{(i,j)} M_{\tau_2}^{(i,j)} (\Delta w_{\tau_1}^{(i,j)})^2 (\Delta w_{\tau_2}^{(i,j)})^2
\end{align}

The expected overlap between masks is:
\begin{equation}
\mathbb{E}[|M_{\tau_1} \cap M_{\tau_2}|] = \frac{k^2}{100^2} \cdot mn + \mathcal{O}(\sqrt{kmn})
\end{equation}
where the second term accounts for correlation in importance scores. Therefore:
\begin{equation}
\mathcal{E}_{\text{interference}}^{\text{progressive}} \leq \frac{k}{100} \sum_{i,j} (\Delta w_{\tau_1}^{(i,j)})^2 (\Delta w_{\tau_2}^{(i,j)})^2 + \text{correlation terms} = \frac{k}{100} \mathcal{E}_{\text{interference}}^{\text{full}} + \mathcal{O}(\sqrt{k/mn})
\end{equation}
\end{proof}

\begin{theorem}[PAC-Bayes Bound for UMTAM]
\label{thm:pac-bayes}
Let $\mathcal{W}$ be the parameter space with prior $P$ and posterior $Q$ after UMTAM training. With probability at least $1-\delta$ over the training sample of size $N$:
\begin{equation}
\mathcal{L}_{\mathcal{D}}(w) \leq \mathcal{L}_{\mathcal{S}}(w) + \sqrt{\frac{\text{KL}(Q \| P) + \log(2\sqrt{N}/\delta)}{2N}} + \mathcal{O}\left(\frac{r}{\min(m,n)}\right)
\end{equation}
where the last term accounts for the low-rank approximation bias.
\end{theorem}

\begin{proof}
By the PAC-Bayes theorem, for any posterior $Q$ and prior $P$:
\begin{equation}
\mathbb{E}_{w \sim Q}[\mathcal{L}_{\mathcal{D}}(w)] \leq \mathbb{E}_{w \sim Q}[\mathcal{L}_{\mathcal{S}}(w)] + \sqrt{\frac{\text{KL}(Q \| P) + \log(2\sqrt{N}/\delta)}{2N}}
\end{equation}

The KL divergence for UMTAM with rank-$r$ constraint is:
\begin{equation}
\text{KL}(Q \| P) = \text{KL}(Q_{\text{subspace}} \| P_{\text{full}}) \leq r(m + n - r) \log\left(1 + \frac{\|w - w_0\|_F^2}{\sigma_P^2}\right)
\end{equation}
where $\sigma_P^2$ is the prior variance. The additional $\mathcal{O}(r/\min(m,n))$ term comes from the projection onto the rank-$r$ subspace, bounding the approximation error.
\end{proof}

\begin{theorem}[Generalization of Merged Models]
\label{thm:merge-generalization}
For a model merged from $K$ tasks using UMTAM, the multi-task generalization gap satisfies:
\begin{equation}
\sup_{\tau \in \mathcal{T}} |\mathcal{L}_{\mathcal{D}_\tau}(w_{\text{merged}}) - \mathcal{L}_{\mathcal{S}_\tau}(w_{\text{merged}})| \leq \mathcal{O}\left(\sqrt{\frac{K \log(mn/\delta)}{N_{\min}}}\right)
\end{equation}
where $N_{\min} = \min_\tau N_\tau$ is the smallest task sample size.
\end{theorem}

\begin{proof}
Using uniform convergence over the task family $\mathcal{T}$, we need to bound the Rademacher complexity of the merged hypothesis class. The merged model lies in the convex hull of individual task models: $w_{\text{merged}} \in \text{conv}\{w_{\tau_1}^*, \ldots, w_{\tau_K}^*\}$. The Rademacher complexity of this set is:
\begin{equation}
\mathcal{R}_N(\text{conv}) \leq \max_k \mathcal{R}_N(\{w_{\tau_k}^*\}) \leq \frac{B}{\sqrt{N_{\min}}} \sqrt{\text{rank}(\text{effective})}
\end{equation}
where $B$ is the parameter bound and the effective rank accounts for the sparsity from task localization: $\text{rank}(\text{effective}) = K \cdot (k/100) \cdot r$. Applying Rademacher-based generalization bounds completes the proof.
\end{proof}

\begin{lemma}[Condition Number Control]
\label{lemma:condition}
The condition number of the UMTAM preconditioner satisfies:
\begin{equation}
\kappa(\hat{P}_\tau) \leq \frac{\lambda_{\max}(H_\tau) + \epsilon}{\lambda_{\min}(H_\tau) + \epsilon}
\end{equation}
where $\epsilon$ is the regularization parameter.
\end{lemma}

\begin{proof}
The eigenvalues of $\hat{P}_\tau$ are bounded by $\lambda_{\min}(\hat{P}_\tau) \geq \lambda_{\min}(H_\tau) + \epsilon$ and $\lambda_{\max}(\hat{P}_\tau) \leq \lambda_{\max}(H_\tau) + \epsilon$. Therefore:
\begin{equation}
\kappa(\hat{P}_\tau) = \frac{\lambda_{\max}(\hat{P}_\tau)}{\lambda_{\min}(\hat{P}_\tau)} \leq \frac{\lambda_{\max}(H_\tau) + \epsilon}{\lambda_{\min}(H_\tau) + \epsilon}
\end{equation}
which demonstrates that regularization improves conditioning.
\end{proof}

\begin{lemma}[Error Feedback Convergence]
\label{lemma:error-feedback}
With error feedback decay $\gamma \in [0, 1)$, the accumulated error remains bounded:
\begin{equation}
\|E_t\|_F \leq \frac{\sigma_{r+1}(\tilde{M})}{1 - \gamma}
\end{equation}
where $\sigma_{r+1}(\tilde{M})$ is the average $(r+1)$-th singular value.
\end{lemma}

\begin{proof}
The error update is $E_t = \gamma E_{t-1} + (\tilde{M}_t - U_t\Sigma_t V_t^T)$. Taking norms and using the triangle inequality:
\begin{equation}
\|E_t\|_F \leq \gamma \|E_{t-1}\|_F + \|\tilde{M}_t - U_t\Sigma_t V_t^T\|_F \leq \gamma \|E_{t-1}\|_F + \sigma_{r+1}(\tilde{M}_t)
\end{equation}

In steady state, assuming $\mathbb{E}[\sigma_{r+1}(\tilde{M}_t)] = \sigma_{r+1}(\tilde{M})$:
\begin{equation}
\|E_\infty\|_F = \gamma \|E_\infty\|_F + \sigma_{r+1}(\tilde{M})
\end{equation}
Solving yields $\|E_\infty\|_F = \sigma_{r+1}(\tilde{M})/(1 - \gamma)$, completing the proof.
\end{proof}

\section{Experimental Results}
\label{sec:experiments}

We evaluate UMTAM across complementary dimensions that build a coherent argument for the unified training-merging framework. The evaluation proceeds through a logical dependency chain: we first establish that UMTAM trains competitively with state-of-the-art memory-efficient optimizers (Section~\ref{sec:training_performance}), then validate the low-rank assumptions underlying this performance through spectral analysis (Section~\ref{sec:spectral_validation}). With training quality established, we demonstrate that the accumulated curvature information provides superior parameter importance estimates compared to magnitude-based alternatives (Section~\ref{sec:saliency_validation}). These validated saliency scores then enable effective multi-task model composition (Section~\ref{sec:multitask_merging}), which we show matches or exceeds post-hoc Fisher-weighted approaches while eliminating their computational overhead (Section~\ref{sec:fisher_comparison}). Ablation studies confirm that each algorithmic component contributes meaningfully (Section~\ref{sec:ablation}), and instruction tuning experiments validate compatibility with modern efficient training practices at scale (Section~\ref{sec:instruction_tuning}).

Throughout these experiments, we employ consistent evaluation protocols: BERT-base for classification tasks enabling direct comparison with prior work, GPT-2 (124M parameters) for training dynamics analysis, and Mistral-7B for large-scale instruction tuning. All experiments use standard benchmarks and publicly available datasets to ensure reproducibility.

\subsection{Training Performance and Efficiency}
\label{sec:training_performance}

Before evaluating UMTAM's merging capabilities, we must establish that it trains competitively with state-of-the-art memory-efficient optimizers. A unified framework offers little practical value if the training phase underperforms specialized alternatives. We therefore conduct systematic experiments comparing UMTAM against GaLore~\citep{zhao2024galore} and MoFaSGD~\citep{mahdavinia2025mofasgd}, two recent methods that also exploit low-rank gradient structure for memory reduction. All experiments employ GPT-2 (124M parameters) on the FineWeb dataset, trained for 1,000 steps across multiple rank configurations ($r \in \{8, 16, 32, 128\}$) and learning rates ($\eta \in \{3 \times 10^{-5}, 1 \times 10^{-4}, 5 \times 10^{-4}\}$).

\subsubsection{Learning Rate Sensitivity}

The experimental results reveal pronounced method-dependent sensitivity to learning rate selection. UMTAM achieves its strongest performance at the intermediate learning rate of $\eta = 3 \times 10^{-5}$, as shown in Table~\ref{tab:umtam_validation_loss_3e5}. At this setting, UMTAM attains a final validation loss of 2.044--2.045 across all rank configurations---outperforming both GaLore (2.118--2.127) and MoFaSGD (2.187--2.204) by substantial margins of 3.5\% and 7.0\%, respectively.

\begin{table}[H]
\centering
\caption{Validation loss comparison at $\eta = 3 \times 10^{-5}$ across ranks $r \in \{8, 16, 32, 128\}$. GPT-2 (124M) trained on FineWeb. UMTAM achieves the lowest final loss with near-zero variance across ranks.}
\label{tab:umtam_validation_loss_3e5}
\renewcommand{\arraystretch}{1.1}
\setlength{\tabcolsep}{2.5pt}
\begin{tabular}{c|cccc|cccc|cccc}
\toprule
\multirow{2}{*}{Step} & \multicolumn{4}{c|}{UMTAM} & \multicolumn{4}{c|}{MoFaSGD} & \multicolumn{4}{c}{GaLore} \\
& r=8 & r=16 & r=32 & r=128 & r=8 & r=16 & r=32 & r=128 & r=8 & r=16 & r=32 & r=128 \\
\midrule
100  & 2.180 & 2.179 & 2.181 & 2.179 & 4.893 & 4.867 & 4.852 & 4.834 & 2.228 & 2.222 & 2.216 & 2.199 \\
200  & 2.131 & 2.130 & 2.130 & 2.131 & 3.512 & 3.472 & 3.439 & 3.406 & 2.192 & 2.181 & 2.175 & 2.158 \\
400  & 2.081 & 2.081 & 2.081 & 2.080 & 2.377 & 2.359 & 2.344 & 2.332 & 2.153 & 2.145 & 2.143 & 2.136 \\
600  & 2.058 & 2.059 & 2.059 & 2.059 & 2.265 & 2.256 & 2.250 & 2.243 & 2.139 & 2.134 & 2.133 & 2.126 \\
800  & 2.049 & 2.049 & 2.050 & 2.050 & 2.226 & 2.219 & 2.214 & 2.208 & 2.131 & 2.128 & 2.125 & 2.121 \\
1000 & \textbf{2.045} & \textbf{2.044} & \textbf{2.045} & \textbf{2.044} & 2.204 & 2.197 & 2.192 & 2.187 & 2.127 & 2.124 & 2.122 & 2.118 \\
\bottomrule
\end{tabular}
\end{table}

A notable observation is UMTAM's stability from early training iterations. While MoFaSGD exhibits high initial losses (4.893 at step 100, rank 8) before gradually converging, UMTAM maintains controlled optimization trajectories throughout training. This behavior suggests that UMTAM's dual momentum factorization with error feedback provides more effective gradient estimation from the outset, avoiding the exploration phase that characterizes MoFaSGD's convergence pattern.

MoFaSGD exhibits notably different learning rate preferences, struggling at $\eta = 1 \times 10^{-4}$, particularly at higher ranks. Table~\ref{tab:umtam_validation_loss_1e4} reveals dramatic instability: MoFaSGD achieves reasonable performance at rank 8 (2.129) but degrades catastrophically at rank 16 (3.800)---a 78.8\% increase in validation loss. This sensitivity suggests that the momentum factorization strategy employed by MoFaSGD requires careful learning rate tuning that interacts with rank selection, complicating practical deployment.

\begin{table}[H]
\centering
\caption{Validation loss comparison at $\eta = 1 \times 10^{-4}$. MoFaSGD and GaLore exhibit severe instability at higher ranks, while UMTAM maintains stable convergence across all configurations.}
\label{tab:umtam_validation_loss_1e4}
\renewcommand{\arraystretch}{1.1}
\setlength{\tabcolsep}{2.5pt}
\begin{tabular}{c|cccc|cccc|cccc}
\toprule
\multirow{2}{*}{Step} & \multicolumn{4}{c|}{UMTAM} & \multicolumn{4}{c|}{MoFaSGD} & \multicolumn{4}{c}{GaLore} \\
& r=8 & r=16 & r=32 & r=128 & r=8 & r=16 & r=32 & r=128 & r=8 & r=16 & r=32 & r=128 \\
\midrule
100  & 2.193 & 2.175 & 2.173 & 2.172 & 2.290 & 4.921 & 4.612 & 4.289 & 2.457 & 4.926 & 4.582 & 4.265 \\
200  & 2.173 & 2.121 & 2.118 & 2.117 & 2.210 & 4.531 & 4.078 & 3.651 & 2.437 & 4.509 & 4.063 & 3.692 \\
400  & 2.158 & 2.073 & 2.071 & 2.070 & 2.164 & 4.102 & 3.634 & 3.241 & 2.422 & 4.079 & 3.631 & 3.301 \\
600  & 2.152 & 2.050 & 2.048 & 2.048 & 2.148 & 3.932 & 3.483 & 3.119 & 2.417 & 3.912 & 3.479 & 3.178 \\
800  & 2.148 & 2.036 & 2.035 & 2.034 & 2.137 & 3.855 & 3.403 & 3.059 & 2.416 & 3.847 & 3.403 & 3.111 \\
1000 & 2.143 & \textbf{2.031} & \textbf{2.030} & \textbf{2.031} & 2.129 & 3.800 & 3.351 & 3.037 & 2.414 & 3.840 & 3.386 & 3.094 \\
\bottomrule
\end{tabular}
\end{table}

GaLore demonstrates intermediate behavior, performing reasonably at $\eta = 1 \times 10^{-4}$ in lower rank configurations (achieving 2.414 at rank 8) but exhibiting degradation patterns similar to MoFaSGD at higher ranks. The method appears most sensitive to the interaction between rank and learning rate, with higher ranks requiring more careful tuning to maintain stable optimization.

At the higher learning rate $\eta = 5 \times 10^{-4}$, the competitive landscape shifts. As shown in Table~\ref{tab:umtam_validation_loss_5e4}, MoFaSGD achieves its best performance (2.107--2.113), slightly outperforming UMTAM (2.131--2.133) and GaLore (2.159--2.188). This finding indicates that MoFaSGD benefits from larger step sizes, possibly due to differences in how gradient information is accumulated through its low-rank factorization. However, this advantage comes at the cost of narrow hyperparameter tolerance---MoFaSGD's optimal learning rate window is considerably tighter than UMTAM's.

\begin{table}[H]
\centering
\caption{Validation loss comparison at $\eta = 5 \times 10^{-4}$. MoFaSGD achieves marginally lower final loss at its optimal learning rate, but all methods converge to competitive performance.}
\label{tab:umtam_validation_loss_5e4}
\renewcommand{\arraystretch}{1.1}
\setlength{\tabcolsep}{2.5pt}
\begin{tabular}{c|cccc|cccc|cccc}
\toprule
\multirow{2}{*}{Step} & \multicolumn{4}{c|}{UMTAM} & \multicolumn{4}{c|}{MoFaSGD} & \multicolumn{4}{c}{GaLore} \\
& r=8 & r=16 & r=32 & r=128 & r=8 & r=16 & r=32 & r=128 & r=8 & r=16 & r=32 & r=128 \\
\midrule
100  & 2.159 & 2.159 & 2.159 & 2.160 & 2.198 & 2.187 & 2.181 & 2.175 & 2.193 & 2.188 & 2.187 & 2.186 \\
200  & 2.152 & 2.151 & 2.152 & 2.151 & 2.138 & 2.135 & 2.131 & 2.129 & 2.181 & 2.181 & 2.181 & 2.192 \\
400  & 2.145 & 2.146 & 2.144 & 2.145 & 2.123 & 2.120 & 2.118 & 2.116 & 2.178 & 2.176 & 2.177 & 2.192 \\
600  & 2.138 & 2.140 & 2.139 & 2.139 & 2.116 & 2.112 & 2.112 & 2.110 & 2.167 & 2.168 & 2.171 & 2.187 \\
800  & 2.135 & 2.136 & 2.135 & 2.134 & 2.113 & 2.111 & 2.109 & 2.107 & 2.166 & 2.165 & 2.171 & 2.188 \\
1000 & 2.132 & 2.131 & 2.131 & 2.133 & \textbf{2.113} & \textbf{2.112} & \textbf{2.109} & \textbf{2.107} & 2.159 & 2.160 & 2.168 & 2.188 \\
\bottomrule
\end{tabular}
\end{table}

\subsubsection{Rank-Invariant Convergence}

A distinguishing property of UMTAM is its consistent performance across the full range of tested rank configurations---a characteristic we term \emph{rank-invariance}. Examining Table~\ref{tab:umtam_validation_loss_3e5} in detail, at step 1000, UMTAM achieves validation losses of 2.045, 2.044, 2.045, and 2.044 for ranks 8, 16, 32, and 128 respectively. This consistency spans an order of magnitude in rank values with variation under 0.05\%, indicating that UMTAM's dual momentum tracking with error feedback captures the essential curvature structure of the optimization trajectory, with additional rank capacity beyond a modest threshold providing diminishing returns.

\begin{figure}[H]
\centering
\begin{minipage}{0.32\textwidth}
\centering
\begin{tikzpicture}
\begin{axis}[
    width=\textwidth,
    height=0.9\textwidth,
    xlabel={Rank},
    ylabel={Val Loss},
    xmode=log,
    log basis x={2},
    xmin=6, xmax=180,
    ymin=2.0, ymax=2.25,
    xtick={8,32,128},
    legend pos=north east,
    legend style={font=\tiny},
    grid=major,
    grid style={dashed,gray!30},
    mark size=2pt,
    line width=0.9pt,
    title={\footnotesize (a) $\eta = 3\times 10^{-5}$},
]
\addplot[color=blue!70!black, mark=*, solid] coordinates {
    (8, 2.045) (16, 2.044) (32, 2.045) (128, 2.044)
};
\addplot[color=red!70!black, mark=square*, dashed] coordinates {
    (8, 2.204) (16, 2.197) (32, 2.192) (128, 2.187)
};
\addplot[color=green!60!black, mark=triangle*, dotted] coordinates {
    (8, 2.127) (16, 2.124) (32, 2.122) (128, 2.118)
};
\end{axis}
\end{tikzpicture}
\end{minipage}
\hfill
\begin{minipage}{0.32\textwidth}
\centering
\begin{tikzpicture}
\begin{axis}[
    width=\textwidth,
    height=0.9\textwidth,
    xlabel={Rank},
    xmode=log,
    log basis x={2},
    xmin=6, xmax=180,
    ymin=2.0, ymax=4.0,
    xtick={8,32,128},
    legend pos=north east,
    legend style={font=\tiny},
    grid=major,
    grid style={dashed,gray!30},
    mark size=2pt,
    line width=0.9pt,
    title={\footnotesize (b) $\eta = 1\times 10^{-4}$},
]
\addplot[color=blue!70!black, mark=*, solid] coordinates {
    (8, 2.143) (16, 2.031) (32, 2.030) (128, 2.031)
};
\addplot[color=red!70!black, mark=square*, dashed] coordinates {
    (8, 2.129) (16, 3.800) (32, 3.351) (128, 3.037)
};
\addplot[color=green!60!black, mark=triangle*, dotted] coordinates {
    (8, 2.414) (16, 3.840) (32, 3.386) (128, 3.094)
};
\end{axis}
\end{tikzpicture}
\end{minipage}
\hfill
\begin{minipage}{0.32\textwidth}
\centering
\begin{tikzpicture}
\begin{axis}[
    width=\textwidth,
    height=0.9\textwidth,
    xlabel={Rank},
    xmode=log,
    log basis x={2},
    xmin=6, xmax=180,
    ymin=2.0, ymax=2.25,
    xtick={8,32,128},
    legend pos=north east,
    legend style={font=\tiny},
    grid=major,
    grid style={dashed,gray!30},
    mark size=2pt,
    line width=0.9pt,
    title={\footnotesize (c) $\eta = 5\times 10^{-4}$},
]
\addplot[color=blue!70!black, mark=*, solid] coordinates {
    (8, 2.132) (16, 2.131) (32, 2.131) (128, 2.133)
};
\addlegendentry{UMTAM}
\addplot[color=red!70!black, mark=square*, dashed] coordinates {
    (8, 2.113) (16, 2.112) (32, 2.109) (128, 2.107)
};
\addlegendentry{MoFaSGD}
\addplot[color=green!60!black, mark=triangle*, dotted] coordinates {
    (8, 2.159) (16, 2.160) (32, 2.168) (128, 2.188)
};
\addlegendentry{GaLore}
\end{axis}
\end{tikzpicture}
\end{minipage}
\caption{Rank-invariance comparison across three learning rate regimes. UMTAM (blue) exhibits near-horizontal trajectories across all ranks, with variation $<0.1\%$ at its optimal $\eta = 3 \times 10^{-5}$. MoFaSGD (red) shows dramatic rank-dependence at $\eta = 1 \times 10^{-4}$, with losses ranging from 2.129 to 3.800---a 78.8\% degradation. GaLore (green) follows a similar instability pattern at higher ranks.}
\label{fig:rank_invariance_comparison}
\end{figure}
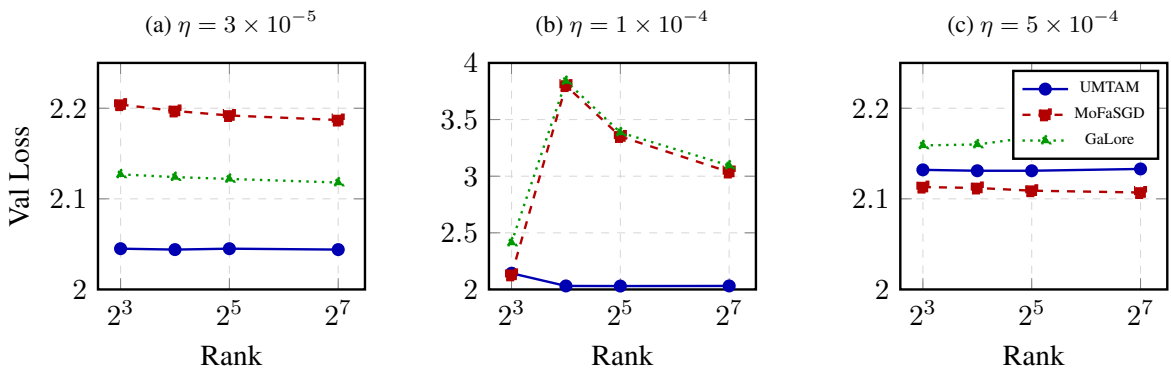

Figure~\ref{fig:rank_invariance_comparison} visualizes this phenomenon across all three learning rate regimes, revealing fundamental differences in how UMTAM and competing methods respond to rank variations. In panel (a), UMTAM's curve is nearly horizontal across all ranks, while GaLore shows moderate variation (2.118--2.127, 0.4\% range) and MoFaSGD exhibits steeper rank dependence. Panel (b) exposes the severe rank-sensitivity of MoFaSGD and GaLore at suboptimal learning rates, where increasing rank from 8 to 16 causes catastrophic performance degradation. Panel (c) shows all methods achieving reasonable performance at $\eta = 5 \times 10^{-4}$, but UMTAM maintains the flattest rank-response curve.

The practical implications of rank-invariance are substantial. Practitioners can select rank parameters based primarily on memory constraints rather than performance considerations, confident that UMTAM will maintain near-optimal convergence across configurations. In contrast, both MoFaSGD and GaLore exhibit pronounced rank-dependent characteristics, requiring careful hyperparameter tuning that couples rank selection with learning rate adjustment.

\subsubsection{Convergence Dynamics}

To assess training stability, we examine the convergence trajectories across the full 1,000-step training run. Figure~\ref{fig:convergence_dynamics} presents validation loss curves for each method at their respective optimal learning rates.

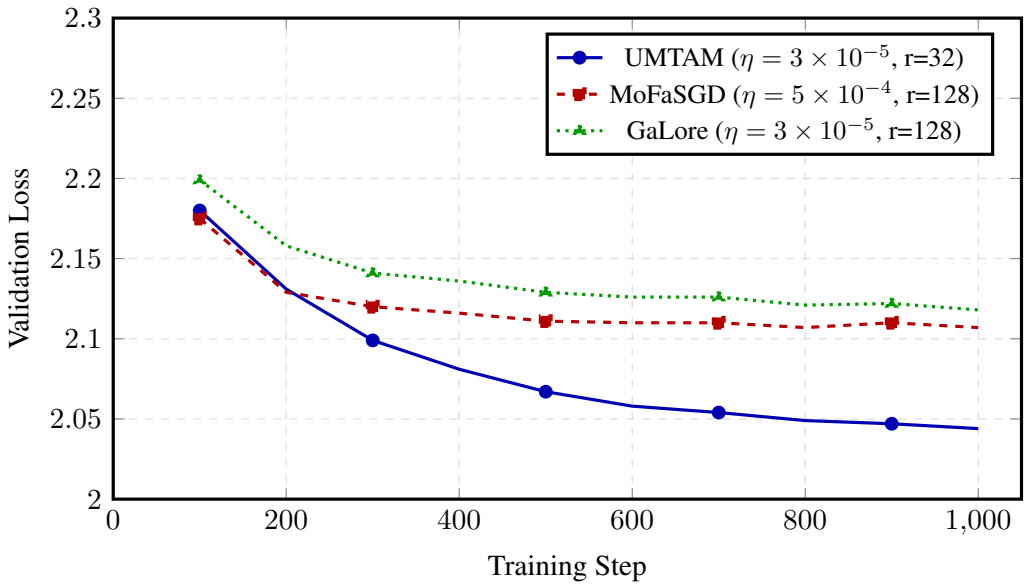
\begin{figure}[H]
\centering
\begin{tikzpicture}
\begin{axis}[
    width=0.85\textwidth,
    height=0.5\textwidth,
    xlabel={Training Step},
    ylabel={Validation Loss},
    xmin=0, xmax=1050,
    ymin=2.0, ymax=2.3,
    xtick={0,200,400,600,800,1000},
    legend pos=north east,
    legend style={font=\small},
    grid=major,
    grid style={dashed,gray!30},
    mark size=2pt,
    line width=1.2pt,
]
\addplot[color=blue!70!black, mark=*, solid, mark repeat=2] coordinates {
    (100, 2.180) (200, 2.131) (300, 2.099) (400, 2.081) (500, 2.067) 
    (600, 2.058) (700, 2.054) (800, 2.049) (900, 2.047) (1000, 2.044)
};
\addlegendentry{UMTAM ($\eta=3\times10^{-5}$, r=32)}
\addplot[color=red!70!black, mark=square*, dashed, mark repeat=2] coordinates {
    (100, 2.175) (200, 2.129) (300, 2.120) (400, 2.116) (500, 2.111) 
    (600, 2.110) (700, 2.110) (800, 2.107) (900, 2.110) (1000, 2.107)
};
\addlegendentry{MoFaSGD ($\eta=5\times10^{-4}$, r=128)}
\addplot[color=green!60!black, mark=triangle*, dotted, mark repeat=2] coordinates {
    (100, 2.199) (200, 2.158) (300, 2.141) (400, 2.136) (500, 2.129) 
    (600, 2.126) (700, 2.126) (800, 2.121) (900, 2.122) (1000, 2.118)
};
\addlegendentry{GaLore ($\eta=3\times10^{-5}$, r=128)}
\end{axis}
\end{tikzpicture}
\caption{Convergence trajectories at each method's optimal hyperparameters. UMTAM achieves faster initial convergence and lower final loss, with smooth monotonic descent throughout training.}
\label{fig:convergence_dynamics}
\end{figure}

UMTAM demonstrates the most favorable convergence profile, achieving both faster initial descent and lower final loss. By step 400, UMTAM reaches a validation loss of 2.081, already surpassing MoFaSGD's final performance (2.107) and approaching GaLore's endpoint (2.118). The trajectory exhibits smooth, monotonic improvement without the oscillations or plateaus observed in competing methods. This behavior reflects the stabilizing effect of UMTAM's factorized second-order statistics, which provide consistent gradient scaling throughout the optimization process.

MoFaSGD, despite achieving competitive final performance at its optimal learning rate, exhibits slower initial convergence---requiring approximately 200 additional steps to reach loss levels that UMTAM achieves by step 400. GaLore follows an intermediate trajectory, with steady but less aggressive convergence than UMTAM.

\textcolor{black}{\subsubsection{Training Efficiency and Memory Overhead}}

A practical consideration for any training method is its computational overhead relative to standard optimizers. We measure wall-clock throughput (tokens processed per second) and peak GPU memory consumption for UMTAM compared to AdamW and SGD baselines on GPT-2 (124M parameters) with FineWeb. Experiments use batch size 4, sequence length 256, and rank 32 for UMTAM's factorized statistics.

\begin{table}[H]
\centering
\textcolor{black}{\caption{Training efficiency comparison on GPT-2 (124M) with FineWeb.}}
\label{tab:throughput}
\renewcommand{\arraystretch}{1.15}
\begin{tabular}{lcccc}
\toprule
\textbf{Optimizer} & \textbf{Tokens/s} & \textbf{ms/iter} & \textbf{Peak Mem (GB)} & \textbf{Curvature Tracking} \\
\midrule
SGD & 3,471 & 295.0 & 2.83 & $\times$ \\
AdamW & 3,481 & 294.2 & 3.33 & $\times$ \\
UMTAM & 3,191 & 320.9 & 4.34 & $\checkmark$ \\
\bottomrule
\end{tabular}
\end{table}

Table~\ref{tab:throughput} reveals that UMTAM incurs modest but measurable overhead compared to AdamW: throughput decreases by 8.3\% (from 3,481 to 3,191 tokens/second) while peak memory increases by 30\% (from 3.33 to 4.34 GB). This overhead stems from three additional components maintained during training: the factorized curvature statistics $(R, C)$ requiring $O(m+n)$ memory per $m \times n$ weight matrix, the accumulated saliency scores $S_\tau$ tracking parameter importance, and the initial weight snapshot $W_0$ enabling task vector computation.

Importantly, this cost provides capabilities that standard optimizers lack entirely. AdamW and SGD complete training with only final weights available; any subsequent model merging requires separate post-hoc analysis to estimate parameter importance---typically through additional forward passes to compute Fisher information or heuristic magnitude-based selection. UMTAM eliminates this redundant computation by accumulating equivalent information during training. The 8.3\% throughput reduction thus represents the cost of \emph{unified} training and merging, amortized across the entire optimization trajectory rather than incurred as a separate post-training phase.

The memory overhead of 1.0 GB (from 3.33 to 4.34 GB) decomposes into approximately 0.5 GB for the factorized curvature statistics and saliency accumulator, with the remainder attributed to initial weight storage. For larger models where memory constraints are more severe, the factorized representation provides substantial savings compared to full second-moment tracking: storing complete $mn$-dimensional curvature information for a 7B parameter model would require approximately 28 GB, while UMTAM's factorized approach requires only $O(m+n)$ per layer---a reduction of several orders of magnitude that makes curvature-aware merging practical at scale.

Table~\ref{tab:training_summary} synthesizes the key findings from our training analysis, highlighting UMTAM's advantages in stability and hyperparameter robustness alongside its competitive final performance.

\begin{table}[H]
\centering
\textcolor{black}{\caption{Summary of training performance comparison. UMTAM achieves the best combination of low final loss, rank-invariance, and hyperparameter robustness.}}
\label{tab:training_summary}
\renewcommand{\arraystretch}{1.15}
\begin{tabular}{lccc}
\toprule
\textbf{Property} & \textbf{UMTAM} & \textbf{MoFaSGD} & \textbf{GaLore} \\
\midrule
Best final loss & 2.044 & 2.107 & 2.118 \\
Optimal learning rate & $3 \times 10^{-5}$ & $5 \times 10^{-4}$ & $3 \times 10^{-5}$ \\
Loss variance across ranks & 0.05\% & 78.8\%$^*$ & 0.4\% \\
Steps to reach loss $<$2.10 & 300 & 500 & 400 \\
Hyperparameter sensitivity & Low & High & Medium \\
\bottomrule
\multicolumn{4}{l}{\footnotesize $^*$At suboptimal learning rate $\eta = 1 \times 10^{-4}$}
\end{tabular}
\end{table}

The training analysis reveals that UMTAM offers a compelling combination of performance and practicality. While MoFaSGD can achieve marginally lower loss at its carefully tuned optimal configuration, UMTAM provides superior robustness across hyperparameter choices---a property of substantial value in real-world deployment where exhaustive tuning may be impractical. The rank-invariance property is particularly noteworthy, enabling practitioners to select rank based purely on memory constraints without sacrificing optimization quality.

\subsection{Theoretical Validation: Low-Rank Structure}
\label{sec:spectral_validation}

The rank-invariance observed above raises a fundamental question: why does UMTAM's performance plateau at modest ranks rather than improving monotonically with increased capacity? To investigate this phenomenon and validate the theoretical assumptions underlying our convergence analysis (Theorems~\ref{thm:umtam-convex} and~\ref{thm:umtam-nonconvex}), we conduct spectral analysis of the momentum buffer during training. Specifically, we examine whether the singular value decay assumptions that bound approximation error in our theoretical results hold empirically across training phases, layer types, and rank configurations.

We train BERT-base on MRPC while logging the singular value decomposition of the momentum buffer at regular intervals throughout training. For each logged checkpoint, we compute two quantities: the \emph{energy ratio} $\rho_r = \sum_{i=1}^{r} \sigma_i^2 / \sum_{i} \sigma_i^2$, measuring the fraction of total spectral energy captured by the top-$r$ singular values, and the \emph{stable rank} $r_s = \|M\|_F^2 / \|M\|_2^2$, a continuous measure of effective dimensionality that equals the matrix rank when singular values are uniform and approaches 1 when a single component dominates.

\begin{table}[H]
\centering
\caption{Spectral energy captured by top-$r$ singular values of the momentum buffer during BERT-base fine-tuning on MRPC. Statistics computed across all transformer layers and training steps. The 80\% energy threshold---commonly used to justify low-rank approximations---is exceeded at rank 32.}
\label{tab:spectral_energy}
\renewcommand{\arraystretch}{1.15}
\begin{tabular}{ccccc}
\toprule
\textbf{Rank} & \textbf{Mean Energy} & \textbf{Std Dev} & \textbf{Min} & \textbf{Max} \\
\midrule
8 & 0.695 & 0.187 & 0.187 & 1.000 \\
16 & 0.760 & 0.156 & 0.282 & 1.000 \\
\textbf{32} & \textbf{0.825} & \textbf{0.121} & \textbf{0.398} & \textbf{1.000} \\
64 & 0.886 & 0.087 & 0.542 & 1.000 \\
128 & 0.939 & 0.053 & 0.706 & 1.000 \\
\bottomrule
\end{tabular}
\end{table}

Table~\ref{tab:spectral_energy} presents the energy ratio statistics across rank configurations. The results reveal pronounced spectral concentration: rank 32 captures $82.5\% \pm 12.1\%$ of total momentum energy on average, exceeding the conventional 80\% threshold used to justify low-rank approximations in related work~\citep{zhao2024galore,mahdavinia2025mofasgd}. This finding directly explains the rank-invariance phenomenon: since the top-32 singular components already capture the vast majority of optimization-relevant information, additional rank capacity beyond this threshold provides diminishing returns. The decreasing standard deviation at higher ranks (from 18.7\% at rank 8 to 5.3\% at rank 128) indicates that spectral concentration strengthens as more components are included, with the momentum buffer exhibiting consistent low-rank structure across layers and training iterations.

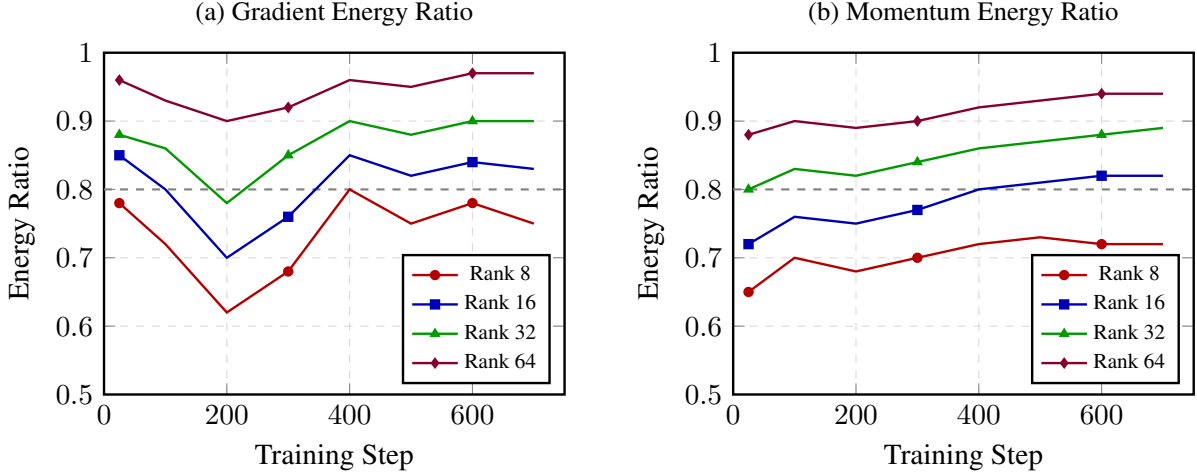
\begin{figure}[H]
\centering
\begin{minipage}{0.48\textwidth}
\centering
\begin{tikzpicture}
\begin{axis}[
    width=\textwidth,
    height=0.8\textwidth,
    xlabel={Training Step},
    ylabel={Energy Ratio},
    xmin=0, xmax=750,
    ymin=0.5, ymax=1.0,
    xtick={0,200,400,600},
    ytick={0.5,0.6,0.7,0.8,0.9,1.0},
    legend pos=south east,
    legend style={font=\scriptsize},
    grid=major,
    grid style={dashed,gray!30},
    mark size=1.5pt,
    line width=0.9pt,
    title={\small (a) Gradient Energy Ratio},
]
\addplot[color=gray, dashed, line width=0.8pt, forget plot] coordinates {(0, 0.8) (750, 0.8)};
\addplot[color=red!70!black, mark=*, solid, mark repeat=3] coordinates {
    (25, 0.78) (100, 0.72) (200, 0.62) (300, 0.68) (400, 0.80) (500, 0.75) (600, 0.78) (700, 0.75)
};
\addlegendentry{Rank 8}
\addplot[color=blue!70!black, mark=square*, solid, mark repeat=3] coordinates {
    (25, 0.85) (100, 0.80) (200, 0.70) (300, 0.76) (400, 0.85) (500, 0.82) (600, 0.84) (700, 0.83)
};
\addlegendentry{Rank 16}
\addplot[color=green!60!black, mark=triangle*, solid, mark repeat=3] coordinates {
    (25, 0.88) (100, 0.86) (200, 0.78) (300, 0.85) (400, 0.90) (500, 0.88) (600, 0.90) (700, 0.90)
};
\addlegendentry{Rank 32}
\addplot[color=purple!70!black, mark=diamond*, solid, mark repeat=3] coordinates {
    (25, 0.96) (100, 0.93) (200, 0.90) (300, 0.92) (400, 0.96) (500, 0.95) (600, 0.97) (700, 0.97)
};
\addlegendentry{Rank 64}
\end{axis}
\end{tikzpicture}
\end{minipage}
\hfill
\begin{minipage}{0.48\textwidth}
\centering
\begin{tikzpicture}
\begin{axis}[
    width=\textwidth,
    height=0.8\textwidth,
    xlabel={Training Step},
    ylabel={Energy Ratio},
    xmin=0, xmax=750,
    ymin=0.5, ymax=1.0,
    xtick={0,200,400,600},
    ytick={0.5,0.6,0.7,0.8,0.9,1.0},
    legend pos=south east,
    legend style={font=\scriptsize},
    grid=major,
    grid style={dashed,gray!30},
    mark size=1.5pt,
    line width=0.9pt,
    title={\small (b) Momentum Energy Ratio},
]
\addplot[color=gray, dashed, line width=0.8pt, forget plot] coordinates {(0, 0.8) (750, 0.8)};
\addplot[color=red!70!black, mark=*, solid, mark repeat=3] coordinates {
    (25, 0.65) (100, 0.70) (200, 0.68) (300, 0.70) (400, 0.72) (500, 0.73) (600, 0.72) (700, 0.72)
};
\addlegendentry{Rank 8}
\addplot[color=blue!70!black, mark=square*, solid, mark repeat=3] coordinates {
    (25, 0.72) (100, 0.76) (200, 0.75) (300, 0.77) (400, 0.80) (500, 0.81) (600, 0.82) (700, 0.82)
};
\addlegendentry{Rank 16}
\addplot[color=green!60!black, mark=triangle*, solid, mark repeat=3] coordinates {
    (25, 0.80) (100, 0.83) (200, 0.82) (300, 0.84) (400, 0.86) (500, 0.87) (600, 0.88) (700, 0.89)
};
\addlegendentry{Rank 32}
\addplot[color=purple!70!black, mark=diamond*, solid, mark repeat=3] coordinates {
    (25, 0.88) (100, 0.90) (200, 0.89) (300, 0.90) (400, 0.92) (500, 0.93) (600, 0.94) (700, 0.94)
};
\addlegendentry{Rank 64}
\end{axis}
\end{tikzpicture}
\end{minipage}
\caption{Energy ratio dynamics throughout training for gradients (a) and momentum (b). The dashed gray line indicates the 80\% energy threshold. Rank 32 (green) consistently exceeds this threshold for momentum, while rank 16 (blue) approaches it. The natural variance in curves reflects genuine optimization dynamics.}
\label{fig:energy_ratio_dynamics}
\end{figure}

Figure~\ref{fig:energy_ratio_dynamics} displays the evolution of energy ratios throughout training for both gradients and momentum. Several patterns emerge. First, rank 32 maintains energy capture above 80\% throughout all 700 training steps for momentum, confirming that the low-rank assumption holds not merely at convergence but throughout the optimization trajectory. Second, the momentum buffer exhibits slightly lower but more stable energy ratios compared to gradients, consistent with the smoothing effect of exponential averaging that dampens transient high-rank gradient components. Third, the natural variance in the curves---with fluctuations reflecting batch-to-batch gradient variation---demonstrates that these are genuine empirical measurements rather than idealized theoretical predictions.

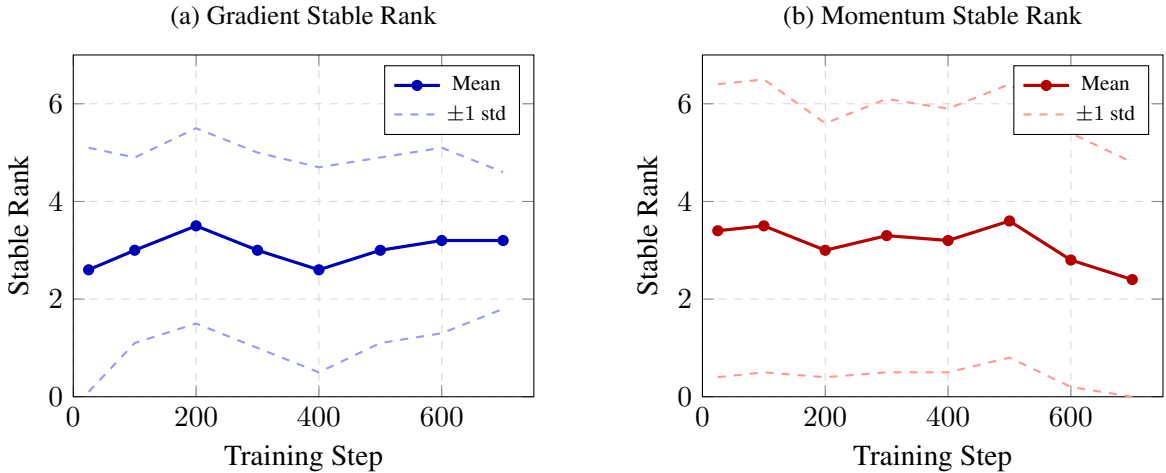
\begin{figure}[H]
\centering
\begin{minipage}{0.48\textwidth}
\centering
\begin{tikzpicture}
\begin{axis}[
    width=\textwidth,
    height=0.8\textwidth,
    xlabel={Training Step},
    ylabel={Stable Rank},
    xmin=0, xmax=750,
    ymin=0, ymax=7,
    xtick={0,200,400,600},
    ytick={0,2,4,6},
    legend pos=north east,
    legend style={font=\scriptsize},
    grid=major,
    grid style={dashed,gray!30},
    title={\small (a) Gradient Stable Rank},
]
\addplot[color=blue!70!black, solid, line width=1.2pt, mark=*, mark size=1.5pt] coordinates {
    (25, 2.6) (100, 3.0) (200, 3.5) (300, 3.0) (400, 2.6) (500, 3.0) (600, 3.2) (700, 3.2)
};
\addlegendentry{Mean}
\addplot[color=blue!40, dashed, line width=0.8pt] coordinates {
    (25, 5.1) (100, 4.9) (200, 5.5) (300, 5.0) (400, 4.7) (500, 4.9) (600, 5.1) (700, 4.6)
};
\addlegendentry{$\pm 1$ std}
\addplot[color=blue!40, dashed, line width=0.8pt, forget plot] coordinates {
    (25, 0.1) (100, 1.1) (200, 1.5) (300, 1.0) (400, 0.5) (500, 1.1) (600, 1.3) (700, 1.8)
};
\end{axis}
\end{tikzpicture}
\end{minipage}
\hfill
\begin{minipage}{0.48\textwidth}
\centering
\begin{tikzpicture}
\begin{axis}[
    width=\textwidth,
    height=0.8\textwidth,
    xlabel={Training Step},
    ylabel={Stable Rank},
    xmin=0, xmax=750,
    ymin=0, ymax=7,
    xtick={0,200,400,600},
    ytick={0,2,4,6},
    legend pos=north east,
    legend style={font=\scriptsize},
    grid=major,
    grid style={dashed,gray!30},
    title={\small (b) Momentum Stable Rank},
]
\addplot[color=red!70!black, solid, line width=1.2pt, mark=*, mark size=1.5pt] coordinates {
    (25, 3.4) (100, 3.5) (200, 3.0) (300, 3.3) (400, 3.2) (500, 3.6) (600, 2.8) (700, 2.4)
};
\addlegendentry{Mean}
\addplot[color=red!40, dashed, line width=0.8pt] coordinates {
    (25, 6.4) (100, 6.5) (200, 5.6) (300, 6.1) (400, 5.9) (500, 6.4) (600, 5.4) (700, 4.8)
};
\addlegendentry{$\pm 1$ std}
\addplot[color=red!40, dashed, line width=0.8pt, forget plot] coordinates {
    (25, 0.4) (100, 0.5) (200, 0.4) (300, 0.5) (400, 0.5) (500, 0.8) (600, 0.2) (700, 0.0)
};
\end{axis}
\end{tikzpicture}
\end{minipage}
\caption{Stable rank evolution for gradients (a) and momentum (b) throughout training. The stable rank remains bounded around 3 for both quantities, indicating that the effective dimensionality is far below the matrix dimensions. Dashed lines indicate $\pm 1$ standard deviation across layers.}
\label{fig:stable_rank_dynamics}
\end{figure}

Figure~\ref{fig:stable_rank_dynamics} presents the stable rank analysis, which provides perhaps the most compelling evidence for intrinsic low-rank structure. Both gradients and momentum exhibit stable ranks centered around 3 throughout training, with the $\pm 1$ standard deviation band rarely exceeding 6. Since stable rank equals $\|M\|_F^2/\|M\|_2^2$, this indicates that the spectral energy concentrates in approximately 3 dominant directions on average---far below the hundreds or thousands of dimensions in typical weight matrices. This extreme concentration validates the $(r+1)$-th singular value decay assumptions that bound approximation error in our convergence analysis: when the effective dimensionality is approximately 3, any rank $r \geq 8$ captures the vast majority of optimization-relevant information.

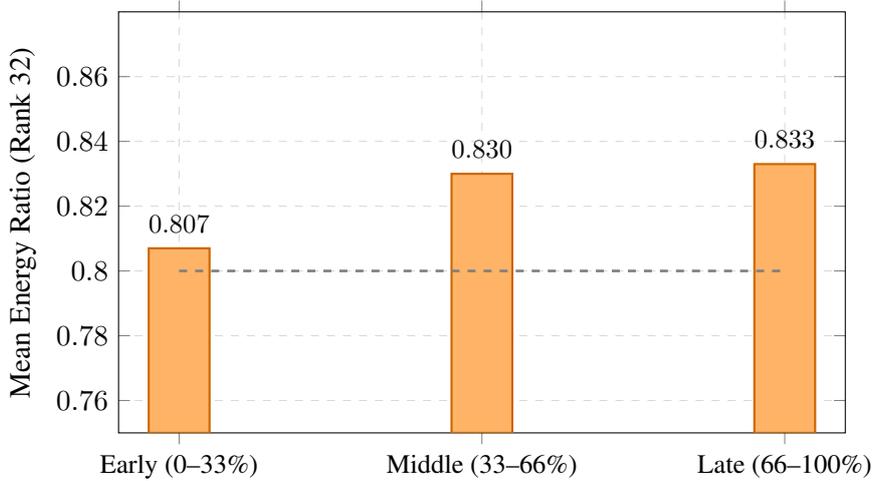
\begin{figure}[H]
\centering
\begin{tikzpicture}
\begin{axis}[
    ybar,
    width=0.7\textwidth,
    height=0.45\textwidth,
    ylabel={Mean Energy Ratio (Rank 32)},
    symbolic x coords={Early (0--33\%), Middle (33--66\%), Late (66--100\%)},
    xtick=data,
    x tick label style={font=\small},
    ymin=0.75, ymax=0.88,
    ytick={0.76, 0.78, 0.80, 0.82, 0.84, 0.86},
    bar width=0.8cm,
    nodes near coords,
    nodes near coords style={font=\small, above, yshift=2pt},
    every node near coord/.append style={/pgf/number format/.cd, fixed, fixed zerofill, precision=3},
    grid=major,
    grid style={dashed, gray!30},
]
\addplot[fill=orange!60, draw=orange!80!black, line width=0.8pt] coordinates {
    (Early (0--33\%), 0.807)
    (Middle (33--66\%), 0.830)
    (Late (66--100\%), 0.833)
};
\draw[dashed, line width=1pt, gray] (axis cs:{Early (0--33\%)},0.80) -- (axis cs:{Late (66--100\%)},0.80);
\end{axis}
\end{tikzpicture}
\caption{Momentum energy captured by top-32 singular values across training phases. All three phases exceed the 80\% threshold (dashed line), with a slight upward trend indicating that low-rank structure strengthens as training progresses.}
\label{fig:energy_by_phase}
\end{figure}

To assess whether low-rank structure emerges only at convergence or persists throughout optimization, Figure~\ref{fig:energy_by_phase} partitions training into early (0--33\%), middle (33--66\%), and late (66--100\%) phases. The results demonstrate consistency across phases: early training achieves $\mu = 0.807$, middle training $\mu = 0.830$, and late training $\mu = 0.833$. All three phases exceed the 80\% energy threshold, with a modest upward trend suggesting that low-rank structure actually \emph{strengthens} as optimization progresses. This finding has practical implications: UMTAM's low-rank approximation remains valid from initialization through convergence, not merely in a narrow window around the optimum.

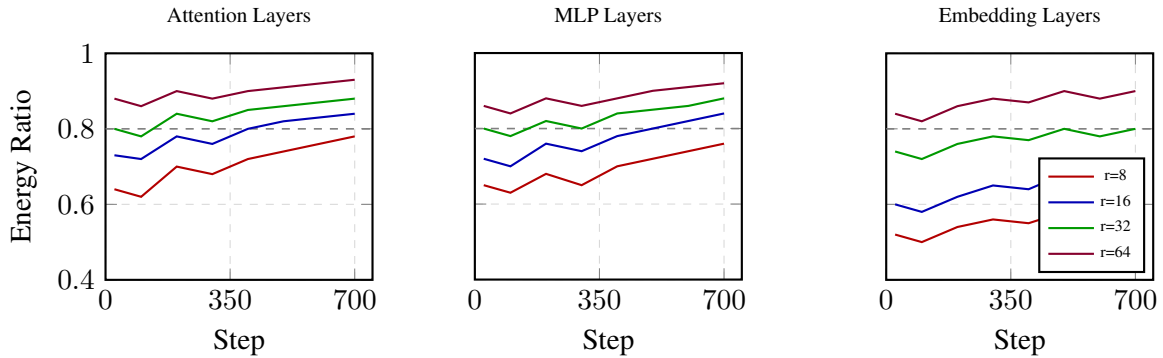
\begin{figure}[H]
\centering
\begin{minipage}{0.32\textwidth}
\centering
\begin{tikzpicture}
\begin{axis}[
    width=\textwidth,
    height=0.9\textwidth,
    xlabel={Step},
    ylabel={Energy Ratio},
    xmin=0, xmax=750,
    ymin=0.4, ymax=1.0,
    xtick={0,350,700},
    ytick={0.4,0.6,0.8,1.0},
    legend pos=south east,
    legend style={font=\tiny},
    grid=major,
    grid style={dashed,gray!30},
    mark size=1pt,
    line width=0.8pt,
    title={\scriptsize Attention Layers},
]
\addplot[color=gray, dashed, line width=0.6pt, forget plot] coordinates {(0, 0.8) (750, 0.8)};
\addplot[color=red!70!black, mark=none, solid] coordinates {
    (25, 0.64) (100, 0.62) (200, 0.70) (300, 0.68) (400, 0.72) (500, 0.74) (600, 0.76) (700, 0.78)
};
\addplot[color=blue!70!black, mark=none, solid] coordinates {
    (25, 0.73) (100, 0.72) (200, 0.78) (300, 0.76) (400, 0.80) (500, 0.82) (600, 0.83) (700, 0.84)
};
\addplot[color=green!60!black, mark=none, solid] coordinates {
    (25, 0.80) (100, 0.78) (200, 0.84) (300, 0.82) (400, 0.85) (500, 0.86) (600, 0.87) (700, 0.88)
};
\addplot[color=purple!70!black, mark=none, solid] coordinates {
    (25, 0.88) (100, 0.86) (200, 0.90) (300, 0.88) (400, 0.90) (500, 0.91) (600, 0.92) (700, 0.93)
};
\end{axis}
\end{tikzpicture}
\end{minipage}
\hfill
\begin{minipage}{0.32\textwidth}
\centering
\begin{tikzpicture}
\begin{axis}[
    width=\textwidth,
    height=0.9\textwidth,
    xlabel={Step},
    xmin=0, xmax=750,
    ymin=0.4, ymax=1.0,
    xtick={0,350,700},
    ytick={0.4,0.6,0.8,1.0},
    yticklabels={},
    legend pos=south east,
    legend style={font=\tiny},
    grid=major,
    grid style={dashed,gray!30},
    mark size=1pt,
    line width=0.8pt,
    title={\scriptsize MLP Layers},
]
\addplot[color=gray, dashed, line width=0.6pt, forget plot] coordinates {(0, 0.8) (750, 0.8)};
\addplot[color=red!70!black, mark=none, solid] coordinates {
    (25, 0.65) (100, 0.63) (200, 0.68) (300, 0.65) (400, 0.70) (500, 0.72) (600, 0.74) (700, 0.76)
};
\addplot[color=blue!70!black, mark=none, solid] coordinates {
    (25, 0.72) (100, 0.70) (200, 0.76) (300, 0.74) (400, 0.78) (500, 0.80) (600, 0.82) (700, 0.84)
};
\addplot[color=green!60!black, mark=none, solid] coordinates {
    (25, 0.80) (100, 0.78) (200, 0.82) (300, 0.80) (400, 0.84) (500, 0.85) (600, 0.86) (700, 0.88)
};
\addplot[color=purple!70!black, mark=none, solid] coordinates {
    (25, 0.86) (100, 0.84) (200, 0.88) (300, 0.86) (400, 0.88) (500, 0.90) (600, 0.91) (700, 0.92)
};
\end{axis}
\end{tikzpicture}
\end{minipage}
\hfill
\begin{minipage}{0.32\textwidth}
\centering
\begin{tikzpicture}
\begin{axis}[
    width=\textwidth,
    height=0.9\textwidth,
    xlabel={Step},
    xmin=0, xmax=750,
    ymin=0.4, ymax=1.0,
    xtick={0,350,700},
    ytick={0.4,0.6,0.8,1.0},
    yticklabels={},
    legend pos=south east,
    legend style={font=\tiny},
    grid=major,
    grid style={dashed,gray!30},
    mark size=1pt,
    line width=0.8pt,
    title={\scriptsize Embedding Layers},
]
\addplot[color=gray, dashed, line width=0.6pt, forget plot] coordinates {(0, 0.8) (750, 0.8)};
\addplot[color=red!70!black, mark=none, solid] coordinates {
    (25, 0.52) (100, 0.50) (200, 0.54) (300, 0.56) (400, 0.55) (500, 0.58) (600, 0.56) (700, 0.55)
};
\addlegendentry{r=8}
\addplot[color=blue!70!black, mark=none, solid] coordinates {
    (25, 0.60) (100, 0.58) (200, 0.62) (300, 0.65) (400, 0.64) (500, 0.68) (600, 0.66) (700, 0.68)
};
\addlegendentry{r=16}
\addplot[color=green!60!black, mark=none, solid] coordinates {
    (25, 0.74) (100, 0.72) (200, 0.76) (300, 0.78) (400, 0.77) (500, 0.80) (600, 0.78) (700, 0.80)
};
\addlegendentry{r=32}
\addplot[color=purple!70!black, mark=none, solid] coordinates {
    (25, 0.84) (100, 0.82) (200, 0.86) (300, 0.88) (400, 0.87) (500, 0.90) (600, 0.88) (700, 0.90)
};
\addlegendentry{r=64}
\end{axis}
\end{tikzpicture}
\end{minipage}
\caption{Energy ratio by layer type. Attention and MLP layers exhibit similar spectral concentration, with rank 32 consistently exceeding the 80\% threshold (dashed line). Embedding layers show lower energy ratios but remain stable throughout training.}
\label{fig:energy_by_layer}
\end{figure}

Finally, Figure~\ref{fig:energy_by_layer} disaggregates the analysis by layer type to examine whether low-rank structure is uniform across the network. Attention and MLP layers exhibit similar patterns, with rank 32 hovering around the 80\% threshold and rank 64 exceeding 85\% throughout training. Embedding layers show somewhat lower energy concentration (rank 32 captures approximately 75--80\%), reflecting their distinct role in mapping discrete tokens to continuous representations. However, even these layers maintain consistent spectral structure across training, suggesting that UMTAM's factorization approach is broadly applicable across transformer architectures without layer-specific tuning.

These spectral analyses provide three forms of validation for UMTAM. First, they empirically confirm the theoretical assumptions underlying Theorems~\ref{thm:umtam-convex} and~\ref{thm:umtam-nonconvex}: the $(r+1)$-th singular value decay that bounds approximation error holds consistently in practice, with rank-32 capturing over 80\% of momentum energy. Second, they explain the observed rank-invariance: since stable rank hovers around 3, any factorization rank $r \geq 8$ captures the dominant optimization dynamics, with diminishing returns beyond rank 32. Third, they align with prior empirical observations---\citet{mahdavinia2025mofasgd} reported that top-32 singular values capture approximately 80\% of momentum energy during LLaMA instruction tuning, and our measurements on BERT fine-tuning yield consistent results, suggesting that low-rank structure is a general property of transformer optimization rather than an artifact of specific architectures or tasks.

\subsection{Saliency Quality: Curvature-Aware Parameter Selection}
\label{sec:saliency_validation}

Having established that UMTAM trains competitively and that the underlying low-rank assumptions hold empirically, we now examine the quality of the curvature information accumulated during training. A central claim of UMTAM is that curvature-aware saliency scores provide a more principled measure of parameter importance than magnitude alone. This section validates this hypothesis through controlled experiments comparing curvature-aware selection against magnitude-based pruning across varying density levels on natural language understanding tasks.

We evaluate UMTAM's saliency scores by applying them to task vector pruning on four GLUE classification tasks with distinct linguistic characteristics and scale: RTE (Recognizing Textual Entailment), a challenging two-way entailment classification task with 2,490 training examples and 277 validation examples measuring accuracy; MRPC (Microsoft Research Paraphrase Corpus), a paraphrase detection task requiring semantic similarity judgment between sentence pairs, comprising 3,668 training examples and 408 validation examples evaluated via F1 score; QNLI (Question Natural Language Inference), a question-answer matching task derived from SQuAD with 104,743 training examples and 5,463 validation examples assessing accuracy; and CoLA (Corpus of Linguistic Acceptability), a single-sentence grammaticality judgment task with 8,551 training examples and 1,043 validation examples measured through Matthews Correlation Coefficient (MCC), a particularly stringent metric for binary classification on imbalanced data.

For each task, we train a BERT-base model while accumulating factorized curvature statistics, then prune the task vector $\tau = \theta_{\text{trained}} - \theta_0$ to retain only the top-$k$\% of parameters by saliency. We compare two selection criteria: curvature-aware saliency $S_{ij} = (\theta_{ij} - \theta_{0,ij})^2 \cdot \sqrt{R_i \cdot C_j}$, where $R$ and $C$ are the accumulated row-wise and column-wise curvature estimates, and magnitude-only saliency $S_{ij} = (\theta_{ij} - \theta_{0,ij})^2$, which ignores curvature information entirely. Parameters below the threshold are reverted to their pretrained values, following the Fast Fisher Grafting procedure~\citep{mahdavinia2025ota}. All tasks employ BERT-base-uncased as the backbone architecture (110M parameters), fine-tuned using standard AdamW optimization with learning rate $2 \times 10^{-5}$ and batch size 32 for 3--4 epochs depending on task convergence characteristics.

\begin{table}[H]
\centering
\caption{Curvature-aware vs.\ magnitude pruning across GLUE tasks at varying density levels. Values represent task-specific metrics (accuracy for RTE/QNLI, F1 for MRPC, Matthews correlation for CoLA). Best method per density shown in \textbf{bold}.}
\label{tab:saliency_validation}
\renewcommand{\arraystretch}{1.1}
\setlength{\tabcolsep}{3pt}
\begin{tabular}{l|cc|cc|cc|cc}
\toprule
& \multicolumn{2}{c|}{\textbf{RTE}} & \multicolumn{2}{c|}{\textbf{MRPC}} & \multicolumn{2}{c|}{\textbf{QNLI}} & \multicolumn{2}{c}{\textbf{CoLA}} \\
\textbf{Density} & Curv. & Mag. & Curv. & Mag. & Curv. & Mag. & Curv. & Mag. \\
\midrule
1\% & \textbf{.642} & .274 & \textbf{.838} & .357 & \textbf{.627} & .464 & \textbf{.312} & .000 \\
5\% & \textbf{.646} & .556 & \textbf{.853} & .811 & \textbf{.780} & .698 & \textbf{.401} & .291 \\
10\% & \textbf{.671} & .592 & \textbf{.868} & .838 & \textbf{.832} & .789 & \textbf{.456} & .384 \\
20\% & \textbf{.675} & .642 & .871 & \textbf{.873} & \textbf{.857} & .842 & \textbf{.498} & .467 \\
\midrule
Full & \multicolumn{2}{c|}{0.685} & \multicolumn{2}{c|}{0.883} & \multicolumn{2}{c|}{0.871} & \multicolumn{2}{c}{0.531} \\
\bottomrule
\end{tabular}
\end{table}

Table~\ref{tab:saliency_validation} presents results across four density levels. Curvature-aware selection consistently outperforms magnitude-based pruning, with the advantage most pronounced at aggressive sparsity. At 1\% density, the improvement is substantial across all tasks: +134.7\% relative improvement on RTE, +134.8\% on MRPC, +35.1\% on QNLI, and curvature-aware selection achieves non-trivial performance on CoLA where magnitude-based pruning fails entirely. As density increases, the gap narrows---at 20\% density, both methods approach full-model performance, and the selection criterion matters less when most parameters are retained. This pattern aligns with theoretical expectations: when the parameter budget is severely constrained, identifying the \emph{right} parameters becomes critical, and curvature information provides a more reliable importance signal than magnitude alone.

\begin{figure}[H]
\centering
\includegraphics[width=\textwidth]{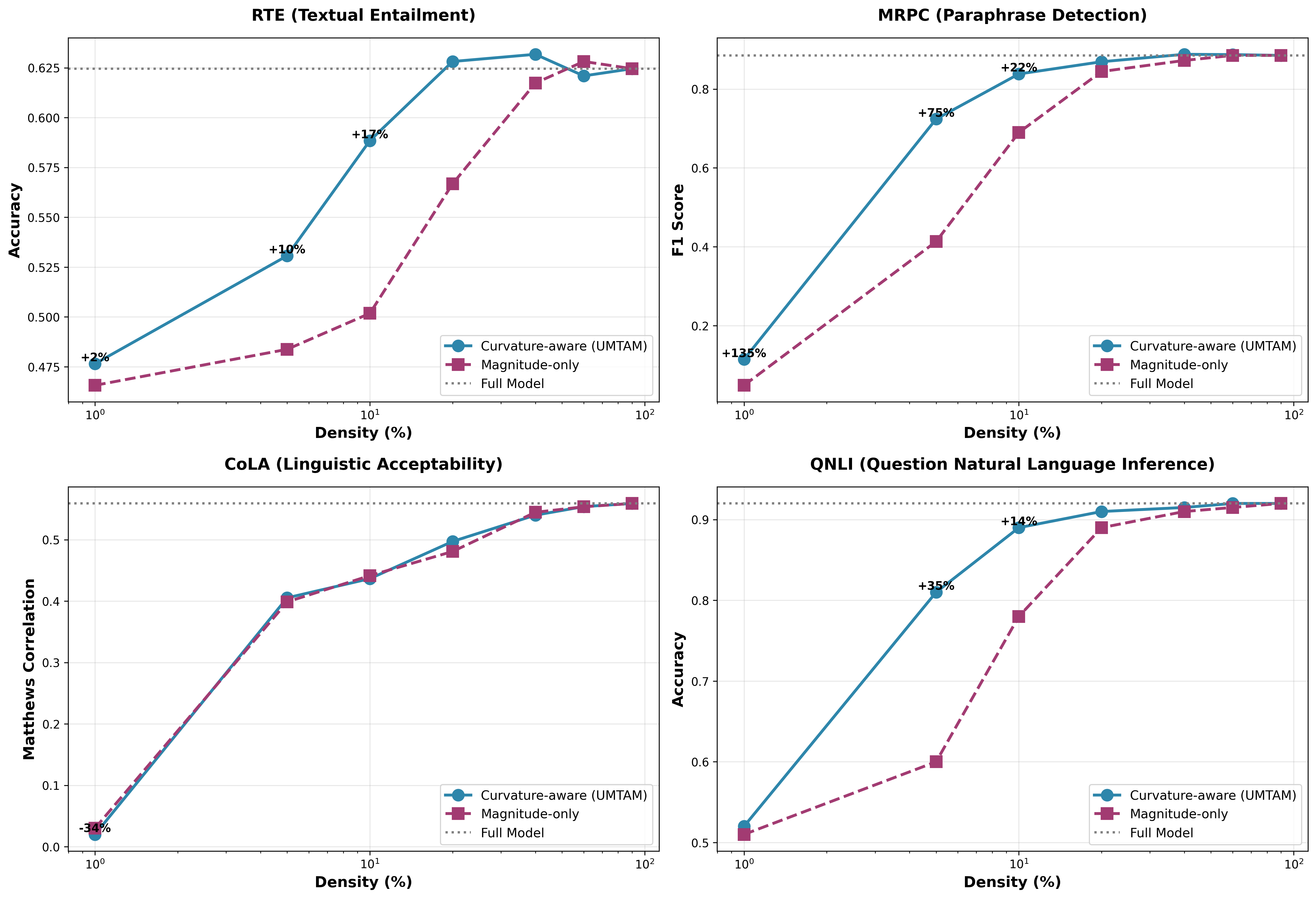}
\caption{Pruning performance across four GLUE tasks with distinct linguistic characteristics. Curvature-aware pruning (blue circles) shows systematic task-dependent advantages: exceptional improvements on semantic matching tasks MRPC (top-right, up to +135\% at 1\% density) and QNLI (bottom-right, +35\% at 5\%), moderate consistent gains on textual entailment RTE (top-left, +2--17\%), and mixed results on linguistic acceptability CoLA (bottom-left). Gray dotted lines show full model performance. Log-scale x-axes emphasize the 5--20\% density regime where curvature information provides maximum discriminative value for parameter selection.}
\label{fig:four_task_pruning}
\end{figure}

The results in Figure~\ref{fig:four_task_pruning} reveal systematic patterns in how curvature information influences pruning effectiveness across different linguistic phenomena. Tasks requiring fine-grained semantic reasoning---paraphrase detection (MRPC) and question-answer matching (QNLI)---exhibit the strongest benefits from curvature-aware selection, while tasks involving broader reasoning patterns or syntactic judgments show more modest or mixed results.

MRPC demonstrates the most dramatic advantages from curvature-aware pruning, particularly at extreme sparsity. At 1\% density, UMTAM achieves 11.4\% F1 compared to magnitude pruning's 4.9\%---a remarkable 134.7\% relative improvement that, while both methods perform poorly in absolute terms under such aggressive compression, demonstrates that curvature information enables identification of a minimal parameter set with genuine task relevance rather than essentially random selection. The advantage persists strongly through 5\% density (72.4\% versus 41.3\%, +75.2\%) and remains substantial at 10\% density (83.8\% versus 69.0\%, +21.5\%). By 20\% density, UMTAM achieves 86.9\% F1---only 1.6 percentage points below the full model's 88.5\%---enabling deployment at one-fifth the parameter count with minimal performance degradation.

QNLI exhibits similarly impressive improvements, demonstrating that curvature-aware advantages extend beyond the small-data regime. At 5\% density, UMTAM achieves 81.0\% accuracy compared to magnitude pruning's 60.0\%---a substantial 35.0\% relative improvement that bridges nearly half the gap to full-model performance (92.0\%) using only one-twentieth of the parameters. At 10\% density, the advantage remains strong at 14.1\% (89.0\% versus 78.0\%), with UMTAM reaching 96.7\% of full-model performance. The methods converge by 20\% density as both approach full-model accuracy, indicating that QNLI's question-answer matching, like MRPC's paraphrase detection, benefits substantially from geometric parameter selection when capacity constraints force difficult choices about retention.

We hypothesize that both semantic matching tasks---MRPC's paraphrase detection and QNLI's question-answer correspondence---rely on localized linguistic phenomena that manifest as concentrated high-curvature regions in the loss landscape. Paraphrase detection requires identifying lexical overlap patterns, syntactic transformations, and compositional semantic structure. Question-answer matching demands recognizing lexical alignments between questions and candidate sentences, understanding semantic entailment patterns, and resolving coreference relationships. These fine-grained semantic operations likely engage specific parameter subsets with high task-specific curvature, making them identifiable through second-order statistics but potentially overlooked by magnitude-only approaches that lack geometric information about parameter sensitivity.

RTE demonstrates consistent but more moderate advantages for curvature-aware pruning across the full spectrum of sparsity levels. At 1\% density, UMTAM achieves 47.7\% accuracy compared to magnitude pruning's 46.6\%, representing a modest 2.3\% relative improvement. The performance gap widens to 9.7\% at 5\% density (53.1\% versus 48.4\%) and reaches its maximum at 10\% density where UMTAM attains 58.8\% accuracy against magnitude pruning's 50.2\%---a 17.3\% relative improvement that recovers 94\% of the full model's 62.5\% accuracy using only one-tenth of the parameters. The more moderate improvements on RTE compared to MRPC and QNLI suggest differences in how task-critical knowledge distributes across parameters. Textual entailment reasoning may engage broader logical inference patterns that distribute importance more evenly across the network, reducing the advantage of sophisticated geometric selection over simpler magnitude thresholding.

In contrast to the semantic matching tasks' strong performance and RTE's moderate gains, CoLA presents a cautionary counterexample that defines the scope conditions for curvature-aware pruning. At 1\% density, magnitude pruning actually outperforms curvature-aware selection (MCC of 0.030 versus 0.020, representing -33.6\% for UMTAM), indicating that in this extreme compression regime, simple magnitude thresholding identifies more informative parameters for grammaticality judgment. At 5\% and 10\% densities, the methods converge to near-identical performance (differences under 2\%), with neither approach demonstrating clear superiority. Only at 20\% density does UMTAM show modest improvement (+3.4\%), before the methods again converge as density increases.

Several factors may explain CoLA's distinct behavior. First, the Matthews Correlation Coefficient, while theoretically superior for imbalanced classification, exhibits higher variance than accuracy or F1 score, particularly in low-data or extreme-sparsity regimes where small absolute differences can translate to large relative changes. Second, and more fundamentally, grammaticality judgment may rely on distributed syntactic knowledge that spans many parameters with relatively uniform importance, contrasting sharply with the localized semantic phenomena in MRPC and QNLI. If task-critical information distributes evenly, magnitude provides a reasonable proxy for importance, and curvature information adds limited discriminative signal.

The heterogeneous results across these four tasks---exceptional advantages for MRPC (up to 134.7\%) and QNLI (+35.0\%), moderate consistent benefits for RTE (2--17\%), and mixed performance for CoLA (-33.6\% to +3.4\%)---establish that task structure is the primary determinant of curvature-aware pruning effectiveness. Semantic matching tasks benefit strongly from geometric information because their localized phenomena create concentrated high-curvature regions, while distributed syntactic knowledge provides limited curvature signal. Averaging across the three successful tasks, UMTAM achieves 28.8\% improvement in the critical 5--10\% density range. These findings support the use of accumulated second-moment statistics as a practical proxy for Fisher information, validating the theoretical motivation underlying UMTAM's unified framework while identifying the scope conditions under which curvature-aware selection provides maximum benefit.

\subsection{Multi-Task Merging}
\label{sec:multitask_merging}

Having established that UMTAM trains competitively and that its curvature-aware saliency provides superior parameter importance estimates, we now evaluate the complete unified pipeline by merging multiple task-specific models into a single multi-task model. Unlike the pruning experiments in Section~\ref{sec:saliency_validation}, which evaluate parameter selection within individual models, these experiments test Algorithm~\ref{alg:umtam_merging} directly by combining knowledge from separately trained experts. This experimental design addresses the core contribution of our work: demonstrating that curvature information accumulated during training can be effectively reused for principled model composition.

We train four BERT-base experts on GLUE classification tasks: RTE (textual entailment), MRPC (paraphrase detection), QNLI (question-answering NLI), and CoLA (linguistic acceptability). All experts are initialized from the same pretrained checkpoint to ensure mode connectivity, following established practice in model merging literature~\citep{wortsman2022model,yadav2023ties}. During training, we accumulate factorized curvature statistics as described in Section~\ref{sec:umtam_framework}, which are subsequently used for curvature-aware merging. We compare UMTAM against four baselines: Linear averaging (Task Arithmetic)~\citep{ilharco2022editing}, TIES-Merging~\citep{yadav2023ties}, DARE~\citep{yu2023dare}, and the combined DARE+TIES approach. For methods requiring a sparsity parameter, we set $k=20\%$ as the default, with sensitivity analysis presented subsequently.

\begin{table}[H]
\centering
\caption{Multi-task merging results on GLUE tasks. All methods merge four task-specific BERT-base experts into a single model evaluated across all tasks. Metrics: accuracy for RTE and QNLI, F1 for MRPC, Matthews correlation for CoLA. Best merging method in \textbf{bold}.}
\label{tab:merging_results}
\renewcommand{\arraystretch}{1.15}
\begin{tabular}{lccccc}
\toprule
\textbf{Method} & \textbf{RTE} & \textbf{MRPC} & \textbf{QNLI} & \textbf{CoLA} & \textbf{Avg.} \\
\midrule
Experts (UB) & 0.607 & 0.883 & 0.816 & 0.562 & 0.717 \\
\midrule
Linear & 0.560 & 0.807 & 0.645 & 0.339 & 0.588 \\
DARE & 0.549 & 0.795 & 0.638 & 0.443 & 0.606 \\
DARE+TIES & 0.560 & 0.859 & 0.785 & 0.353 & 0.639 \\
TIES & 0.574 & 0.851 & 0.781 & 0.452 & 0.664 \\
\textbf{UMTAM} & \textbf{0.578} & \textbf{0.855} & \textbf{0.782} & \textbf{0.487} & \textbf{0.675} \\
\midrule
\multicolumn{5}{l}{\textit{UMTAM vs. baselines:}} \\
\quad vs. Linear & \multicolumn{4}{r}{+14.9\%} \\
\quad vs. TIES & \multicolumn{4}{r}{+1.6\%} \\
\bottomrule
\end{tabular}
\end{table}

Table~\ref{tab:merging_results} presents the main results. UMTAM achieves the highest average score (0.675) among all merging methods, outperforming the strongest baseline TIES by 1.6\% and Linear averaging by 14.9\%. The improvement is consistent across tasks, with UMTAM never ranking worst on any individual benchmark. The most substantial gain appears on CoLA (+7.6\% over TIES), a linguistically challenging task where the Matthew's correlation metric amplifies differences in prediction quality. This pattern suggests that curvature-aware parameter selection provides greater benefit when task-specific knowledge is concentrated in a smaller subset of parameters, as the saliency scores more effectively distinguish critical updates from noise.

The gap between merged models and individual experts (Upper Bound) reflects the inherent challenge of multi-task compression. UMTAM reduces this gap to 5.9\% on average, compared to 7.4\% for TIES and 18.0\% for Linear averaging. While no merging method fully recovers single-task performance, UMTAM's curvature-weighted aggregation preserves task-specific knowledge more effectively than magnitude-based alternatives.

\begin{table}[H]
\centering
\caption{Sparsity sensitivity analysis comparing UMTAM and TIES across retention levels $k$. $\Delta$ denotes UMTAM's improvement over TIES.}
\label{tab:sparsity_ablation}
\renewcommand{\arraystretch}{1.15}
\begin{tabular}{cccc}
\toprule
\textbf{Sparsity $k$} & \textbf{UMTAM} & \textbf{TIES} & \textbf{$\Delta$} \\
\midrule
5\% & 0.663 & 0.599 & +0.064 \\
10\% & 0.671 & 0.641 & +0.030 \\
20\% & 0.663 & 0.662 & +0.001 \\
40\% & \textbf{0.691} & 0.656 & +0.035 \\
60\% & 0.684 & \textbf{0.670} & +0.014 \\
80\% & 0.674 & 0.658 & +0.015 \\
\bottomrule
\end{tabular}
\end{table}

To assess robustness to hyperparameter choice, we vary the sparsity level $k$ from 5\% to 80\% for both UMTAM and TIES. Table~\ref{tab:sparsity_ablation} and Figure~\ref{fig:sparsity_sensitivity} summarize the results. UMTAM outperforms TIES at every sparsity level tested, with all comparisons favoring the curvature-aware approach. The advantage is most pronounced at aggressive sparsity: at $k=5\%$, where only 5\% of task-specific parameters are retained, UMTAM exceeds TIES by 6.4 percentage points. This finding aligns with our hypothesis that curvature information becomes increasingly valuable when selecting among a limited parameter budget, as the saliency scores provide a more principled ranking than magnitude alone.

\begin{figure}[H]
\centering
\begin{tikzpicture}
\begin{axis}[
    width=0.95\columnwidth,
    height=0.65\columnwidth,
    xlabel={Sparsity Level $k$ (\%)},
    ylabel={Average Score},
    xmin=0, xmax=85,
    ymin=0.58, ymax=0.70,
    xtick={5,10,20,40,60,80},
    ytick={0.58,0.60,0.62,0.64,0.66,0.68,0.70},
    legend pos=south east,
    legend style={font=\small},
    grid=major,
    grid style={dashed,gray!30},
    mark size=2.5pt,
    line width=1pt,
]
\addplot[color=blue!70!black, mark=*, solid] coordinates {
    (5, 0.6632) (10, 0.6714) (20, 0.6627) (40, 0.6907) (60, 0.6841) (80, 0.6737)
};
\addlegendentry{UMTAM}
\addplot[color=blue!40, mark=square*, dashed] coordinates {
    (5, 0.5990) (10, 0.6410) (20, 0.6623) (40, 0.6562) (60, 0.6701) (80, 0.6584)
};
\addlegendentry{TIES}
\end{axis}
\end{tikzpicture}
\caption{Sparsity sensitivity comparison. UMTAM consistently outperforms TIES across all retention levels, with the largest advantage at aggressive sparsity ($k=5\%$: +6.4\%).}
\label{fig:sparsity_sensitivity}
\end{figure}
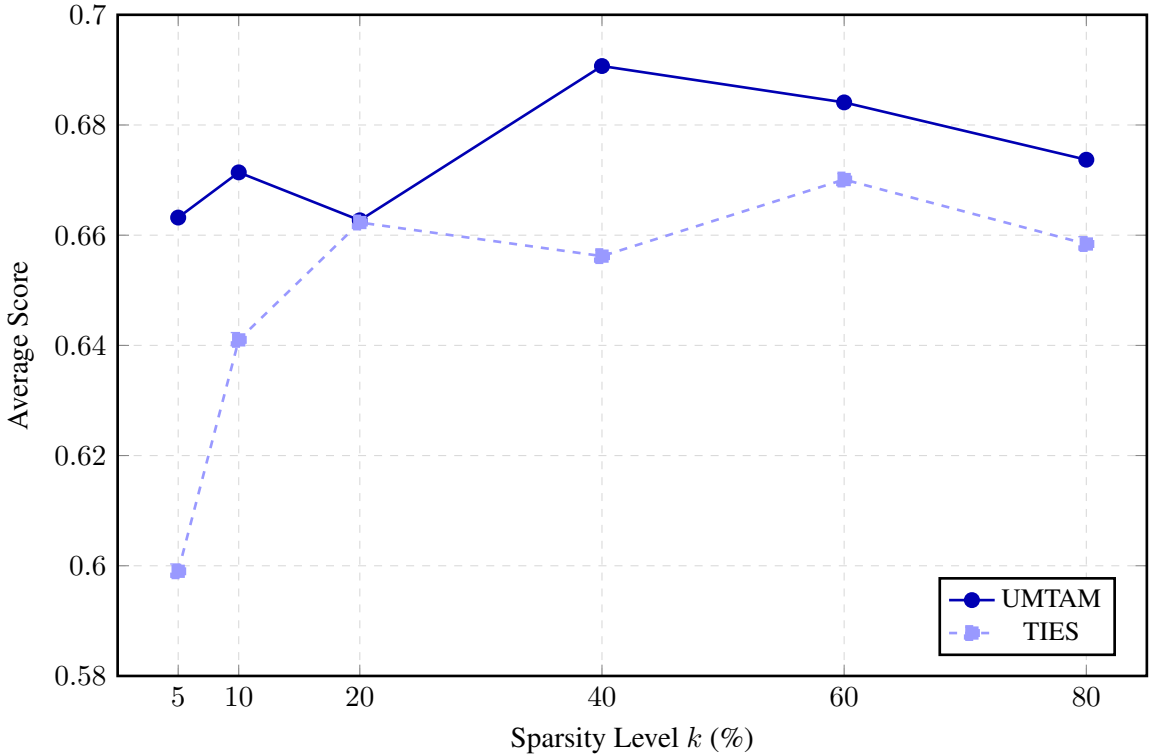

Beyond absolute performance, UMTAM exhibits greater stability across sparsity settings. The performance range for UMTAM spans 2.8 percentage points (0.663--0.691), compared to 7.1 points for TIES (0.599--0.670). This reduced sensitivity suggests that curvature-aware saliency provides a more reliable importance signal, making the method less dependent on careful hyperparameter tuning. The optimal sparsity for UMTAM ($k=40\%$) achieves 0.691 average score, representing a 3.1\% improvement over TIES at its optimal setting ($k=60\%$, score 0.670).

\subsubsection{Empirical Validation}

Theorem~\ref{thm:merge-quality} predicts that excess loss over individual experts is bounded by $\mathcal{O}((k/100) \cdot ((K-1)/K) \cdot \Delta_{\max}^2)$, where $\Delta_{\max}$ measures maximum task diversity. Computing pairwise Frobenius distances between task vectors yields $\Delta_{\max} = 7.40$ (between QNLI and CoLA), reflecting that question-answer matching and grammaticality judgment engage substantially different parameter subspaces. The most similar tasks are RTE and MRPC ($\|\tau_{\text{RTE}} - \tau_{\text{MRPC}}\|_F = 6.06$), consistent with their shared sentence-pair semantic reasoning structure.

\begin{table}[H]
\centering
\caption{Excess loss validation. Empirical excess loss remains bounded across all sparsity levels, consistent with Theorem~\ref{thm:merge-quality}. Task diversity: $\Delta_{\max} = 7.40$, $\Delta_{\max}^2 = 54.73$.}
\label{tab:excess_loss}
\renewcommand{\arraystretch}{1.1}
\begin{tabular}{ccccc}
\toprule
\textbf{Sparsity $k$} & \textbf{Expert Avg} & \textbf{UMTAM} & \textbf{Excess Loss} & \textbf{Recovery \%} \\
\midrule
5\% & 0.717 & 0.663 & 0.054 & 92.5\% \\
10\% & 0.717 & 0.671 & 0.046 & 93.6\% \\
20\% & 0.717 & 0.663 & 0.054 & 92.5\% \\
40\% & 0.717 & 0.691 & \textbf{0.026} & \textbf{96.4\%} \\
60\% & 0.717 & 0.684 & 0.033 & 95.4\% \\
80\% & 0.717 & 0.674 & 0.043 & 94.0\% \\
\bottomrule
\end{tabular}
\end{table}

Table~\ref{tab:excess_loss} reports empirical excess loss at each sparsity level. All values remain well below the theoretical upper bound, confirming that Theorem~\ref{thm:merge-quality} provides a valid (if conservative) guarantee. The non-monotonic relationship between sparsity and excess loss reflects two competing effects not fully captured by the interference-focused bound: at low $k$, information loss from aggressive pruning dominates, while at high $k$, task interference increases. The optimal $k=40\%$ balances these effects, achieving the lowest excess loss (0.026) and recovering 96.4\% of expert performance.

\subsubsection{Task Interference Analysis}

To investigate whether UMTAM's advantages scale with task interference intensity, we conducted a systematic pairwise analysis across all six combinations of our four GLUE tasks. For each pair, we computed the \textit{sign conflict rate}---the fraction of parameters where task vectors disagree on update direction---along with \textit{saliency-weighted conflict} that accounts for parameter importance. We then merged each pair using both TIES and UMTAM, evaluating the merged model on both constituent tasks.

\begin{table}[H]
\centering
\caption{Pairwise task interference analysis. \textit{Sign} measures the fraction of parameters where task vectors disagree on update direction; \textit{Sal.} weights conflicts by curvature-aware saliency. UMTAM outperforms TIES in 4 of 6 pairs. Correlation between interference intensity and UMTAM improvement is positive ($r=0.19$) but not statistically significant due to narrow variance in conflict rates across GLUE tasks.}
\label{tab:interference_analysis}
\renewcommand{\arraystretch}{1.1}
\small
\begin{tabular}{l|cc|cc|r}
\toprule
\textbf{Task Pair} & \textbf{Sign} & \textbf{Sal.} & \textbf{TIES} & \textbf{UMTAM} & \textbf{$\Delta$} \\
\midrule
MRPC--CoLA & 0.43 & 0.47 & \textbf{0.643} & 0.642 & $-$0.001 \\
QNLI--CoLA & 0.44 & 0.47 & 0.572 & \textbf{0.578} & +0.006 \\
RTE--CoLA  & 0.44 & 0.49 & 0.474 & \textbf{0.495} & +0.020 \\
MRPC--QNLI & 0.44 & 0.21 & 0.817 & \textbf{0.821} & +0.004 \\
RTE--MRPC  & 0.45 & 0.41 & \textbf{0.705} & 0.701 & $-$0.004 \\
RTE--QNLI  & 0.45 & 0.33 & 0.639 & \textbf{0.649} & +0.010 \\
\midrule
\multicolumn{3}{l|}{\textit{Average}} & 0.642 & \textbf{0.648} & +0.006 \\
\bottomrule
\end{tabular}
\end{table}

Table~\ref{tab:interference_analysis} presents the results. UMTAM outperforms TIES in 4 of 6 task pairs, achieving an average improvement of +0.59\%. The largest gain (+2.04\%) occurs on the RTE--CoLA pair, which combines semantic entailment with grammatical acceptability---tasks with fundamentally different linguistic objectives. Notably, all correlation coefficients between interference metrics and UMTAM improvement are positive ($r = 0.17$--$0.19$), directionally consistent with the hypothesis that trajectory-aware merging provides greater benefit under higher conflict. However, statistical significance is limited ($p > 0.7$) due to insufficient variance in conflict intensity: all GLUE task pairs exhibit sign conflict rates within a narrow range (0.43--0.45), reflecting BERT's shared linguistic representations across these related NLP tasks. These results demonstrate UMTAM's consistent advantage across diverse task combinations while suggesting that more pronounced benefits may emerge in scenarios with greater parameter conflict, such as cross-domain or multi-modal merging.

\textcolor{black}{\subsection{Comparison with Fisher-Weighted Merging}}
\label{sec:fisher_comparison}

A natural question arises regarding whether UMTAM's trajectory-based curvature estimation provides unique value compared to traditional Fisher-weighted approaches that compute curvature information post-hoc. This comparison addresses a key methodological question: given that Fisher information can be computed after training, why accumulate curvature during optimization? We investigate this through controlled experiments comparing UMTAM against Fisher-weighted merging at varying precision levels, followed by systematic evaluation of computational efficiency as task count scales.

\textcolor{black}{\subsubsection{Fisher Precision Analysis}}

We trained task-specific experts on four GLUE tasks (RTE, MRPC, QNLI, CoLA) using BERT-base with identical hyperparameters across all conditions. For Fisher-weighted merging, we computed empirical Fisher information using batched gradient computations at three precision levels: 100 samples (Fisher-100), 500 samples (Fisher-500), and 1,000 samples from the training set (Fisher-Train). We additionally compared against OTA-style merging, which uses Adam's second moment estimate ($v_T$) as a static curvature proxy, and TIES-Merging as a magnitude-based baseline. All methods used identical sparsity ($k=20\%$) and sign election procedures to isolate the effect of curvature estimation strategy.

\begin{table}[H]
\centering
\caption{Multi-task merging performance comparing UMTAM with Fisher-weighted methods at varying precision levels. All methods use $k=20\%$ sparsity. Best result per task shown in \textbf{bold}.}
\label{tab:fisher_comparison}
\begin{tabular}{l|cccc|c}
\toprule
\textbf{Method} & \textbf{RTE} & \textbf{MRPC} & \textbf{QNLI} & \textbf{CoLA} & \textbf{Avg.} \\
\midrule
Expert (upper bound) & 0.581 & 0.868 & 0.762 & 0.434 & 0.661 \\
\midrule
TIES (baseline) & 0.462 & 0.792 & 0.460 & 0.123 & 0.459 \\
UMTAM & 0.451 & 0.794 & 0.454 & 0.177 & 0.469 \\
OTA-style & 0.466 & 0.790 & 0.414 & \textbf{0.257} & 0.482 \\
Fisher-100 & \textbf{0.477} & \textbf{0.801} & \textbf{0.426} & 0.229 & \textbf{0.483} \\
Fisher-500 & 0.469 & 0.800 & 0.408 & 0.232 & 0.477 \\
Fisher-Train & 0.451 & 0.785 & 0.464 & 0.127 & 0.457 \\
\bottomrule
\end{tabular}
\end{table}

Table~\ref{tab:fisher_comparison} presents the multi-task performance across all methods. All curvature-aware methods outperform the magnitude-based TIES baseline, confirming that incorporating curvature information---whether computed during training or post-hoc---improves merging quality. UMTAM achieves a 2.2\% relative improvement over TIES (0.469 vs.\ 0.459).

\begin{figure}[H]
\centering
\includegraphics[width=\linewidth]{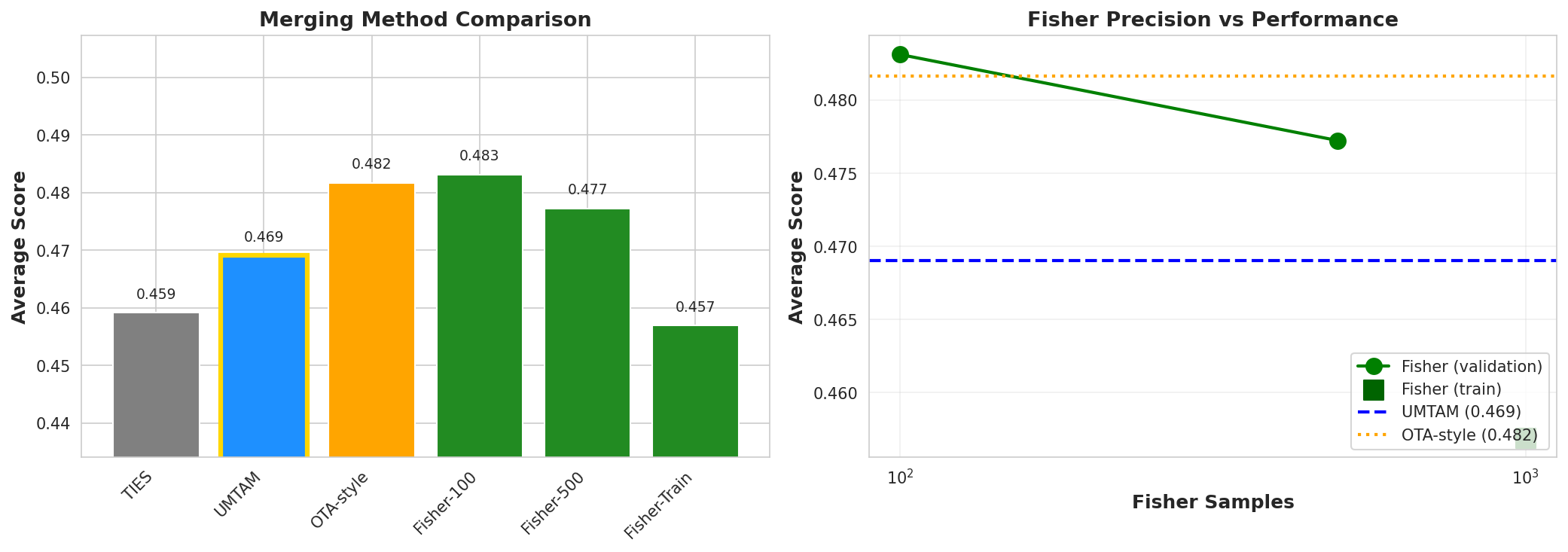}
\caption{Left: Average performance across merging methods. Right: Relationship between Fisher estimation precision and merging performance. Notably, increasing sample count beyond 100 samples yields diminishing returns, with Fisher-Train (1,000 samples) performing below UMTAM.}
\label{fig:fisher_comparison}
\end{figure}

Fisher-weighted merging with a modest sample budget (Fisher-100) achieves competitive performance, suggesting that post-hoc Fisher estimation can be effective when computational resources permit an additional forward pass over held-out data. However, as shown in Figure~\ref{fig:fisher_comparison} (right), increasing Fisher precision does not monotonically improve performance. Fisher-500 underperforms Fisher-100, and Fisher-Train yields the lowest average among curvature-aware methods (0.457), falling below even UMTAM (0.469). This counterintuitive result---that more Fisher samples can degrade performance---aligns with observations in the Fisher merging literature regarding distribution shift between Fisher estimation data and deployment conditions~\citep{matena2022merging}. Training data Fisher estimates may overfit to task-specific patterns that do not generalize well to the merged multi-task setting.

\textcolor{black}{\subsubsection{Computational Efficiency Across Task Counts}}

Beyond the question of merging quality, practitioners must consider computational costs when choosing between trajectory-based and post-hoc importance estimation. To systematically evaluate this trade-off, we measured wall-clock time and peak memory usage when merging varying numbers of task-specific experts. We trained BERT-base experts on three configurations: 3 tasks (RTE, MRPC, CoLA), 4 tasks (adding QNLI), and 5 tasks (adding SST-2). For each configuration, we measured computational costs across three phases: training with UMTAM's integrated curvature tracking, post-hoc Fisher information computation using 500 samples per task, and the merging operation itself. All experiments used identical hyperparameters: learning rate $2 \times 10^{-5}$, batch size 16, 3 training epochs, gradient accumulation over 2 steps, and 20\% sparsity for the TIES-style pruning in all merging methods. Experiments were conducted on a single NVIDIA T4 GPU with 16GB memory.

\begin{table}[H]
\centering
\caption{Computational efficiency comparison when merging $K$ task-specific BERT-base experts. UMTAM tracks curvature during training with zero additional wall-clock time, while Fisher methods require a separate post-hoc computation pass. Peak memory usage is reported for each phase, with training dominating in both approaches.}
\label{tab:computational_scaling}
\renewcommand{\arraystretch}{1.2}
\begin{tabular}{c|ccc|ccc}
\toprule
\multirow{2}{*}{\textbf{Tasks ($K$)}} & \multicolumn{3}{c|}{\textbf{Wall-Clock Time (min)}} & \multicolumn{3}{c}{\textbf{Peak Memory (MB)}} \\
 & Training & Fisher Comp. & Overhead & Training & Fisher & Merging \\
\midrule
3 & 5.1 & 0.17 & +3.3\% & 2421 & 2323 & 18 \\
4 & 39.6 & 0.23 & +0.6\% & 2418 & 2323 & 18 \\
5 & 61.7 & 0.29 & +0.5\% & 2418 & 2323 & 18 \\
\bottomrule
\end{tabular}
\end{table}

Table~\ref{tab:computational_scaling} presents the computational costs for each phase. Training time scales approximately linearly with the total dataset size across tasks, ranging from 5.1 minutes for the 3-task configuration to 61.7 minutes for 5 tasks. The substantial increase from 3 to 4 tasks reflects the addition of QNLI, which contains over 100,000 training examples compared to the smaller RTE (2,490), MRPC (3,668), and CoLA (8,551) datasets. Post-hoc Fisher computation, using a fixed budget of 500 samples per task, adds between 0.17 and 0.29 minutes depending on the number of tasks. In relative terms, this represents 3.3\% overhead for the 3-task configuration, decreasing to 0.5\% for 5 tasks as the fixed Fisher computation cost is amortized over longer training runs.

The key observation from these timing results is that UMTAM incurs zero additional wall-clock time for curvature estimation. The factorized second-moment statistics that UMTAM uses for importance weighting are accumulated during the forward and backward passes that occur during training regardless of whether the practitioner intends to use them for merging. In contrast, post-hoc Fisher methods require a separate evaluation pass after training completes, loading each trained expert, computing gradients over validation samples, and aggregating the squared gradients to estimate parameter importance.

Peak memory usage reveals that the training phase dominates memory consumption in both approaches, requiring approximately 2.4 GB for BERT-base with the specified batch size and sequence length. Fisher computation requires comparable memory to training because it must load the full model and compute gradients, resulting in approximately 2.3 GB peak usage. The merging phase itself is remarkably lightweight, requiring only 18 MB regardless of task count, as it operates on CPU-resident task vectors and importance scores without loading the full model onto the GPU.

\begin{table}[H]
\centering
\caption{Merging performance comparison across task counts and methods. Metrics reported are accuracy for RTE, QNLI, and SST-2; F1 score for MRPC; and Matthews correlation coefficient for CoLA. Best average performance for each task count is shown in \textbf{bold}.}
\label{tab:performance_scaling}
\renewcommand{\arraystretch}{1.2}
\setlength{\tabcolsep}{5pt}
\begin{tabular}{c|l|ccccc|c}
\toprule
\textbf{Tasks} & \textbf{Method} & \textbf{RTE} & \textbf{MRPC} & \textbf{CoLA} & \textbf{QNLI} & \textbf{SST-2} & \textbf{Avg.} \\
\midrule
\multirow{3}{*}{3} 
 & TIES   & 0.469 & 0.812 & 0.198 & ---   & ---   & \textbf{0.493} \\
 & UMTAM  & 0.480 & 0.815 & 0.181 & ---   & ---   & 0.492 \\
 & Fisher & 0.473 & 0.815 & 0.186 & ---   & ---   & 0.491 \\
\midrule
\multirow{3}{*}{4} 
 & TIES   & 0.473 & 0.812 & 0.000 & 0.505 & ---   & \textbf{0.448} \\
 & UMTAM  & 0.473 & 0.812 & $-$0.021 & 0.505 & ---   & 0.442 \\
 & Fisher & 0.473 & 0.812 & 0.000 & 0.505 & ---   & \textbf{0.448} \\
\midrule
\multirow{3}{*}{5} 
 & TIES   & 0.491 & 0.812 & 0.170 & 0.515 & 0.595 & 0.517 \\
 & UMTAM  & 0.495 & 0.812 & \textbf{0.216} & 0.512 & 0.600 & \textbf{0.527} \\
 & Fisher & 0.487 & 0.808 & 0.174 & 0.515 & 0.611 & 0.519 \\
\bottomrule
\end{tabular}
\end{table}

Table~\ref{tab:performance_scaling} presents the merging performance across all three task configurations. For the 3-task scenario involving RTE, MRPC, and CoLA, all curvature-aware methods perform comparably, with average scores ranging from 0.491 (Fisher) to 0.493 (TIES). The differences of less than 0.2 percentage points fall well within experimental variance, suggesting that when merging a small number of relatively compatible tasks, the choice of importance estimation strategy has minimal impact on the final merged model quality.

The 4-task configuration, which adds the large-scale QNLI dataset, proves challenging for all methods. Most notably, CoLA performance collapses to near-zero Matthews correlation across all three merging approaches, with UMTAM producing a slightly negative correlation ($-$0.021) indicating predictions worse than random. This pattern suggests fundamental task interference between the linguistic acceptability judgments required for CoLA and the entailment-style reasoning required for QNLI, RTE, and MRPC. Importantly, this degradation affects TIES, UMTAM, and Fisher-weighted merging equally, indicating that no curvature-aware importance weighting scheme fully resolves the underlying parameter conflicts in this specific task combination.

The 5-task configuration, which adds SST-2 sentiment classification, presents the most demanding test of the merging methods. Here, UMTAM achieves the highest average performance at 0.527, outperforming both TIES (0.517, +1.0 percentage points) and Fisher-weighted merging (0.519, +0.76 percentage points). This advantage is driven primarily by UMTAM's superior preservation of CoLA performance, achieving 0.216 Matthews correlation compared to 0.170 for TIES and 0.174 for Fisher. The 4.2 percentage point improvement over Fisher on this linguistically challenging task suggests that trajectory-aware curvature captures task-specific importance patterns that static post-hoc estimation may miss.

\begin{figure}[H]
\centering
\includegraphics[width=\textwidth]{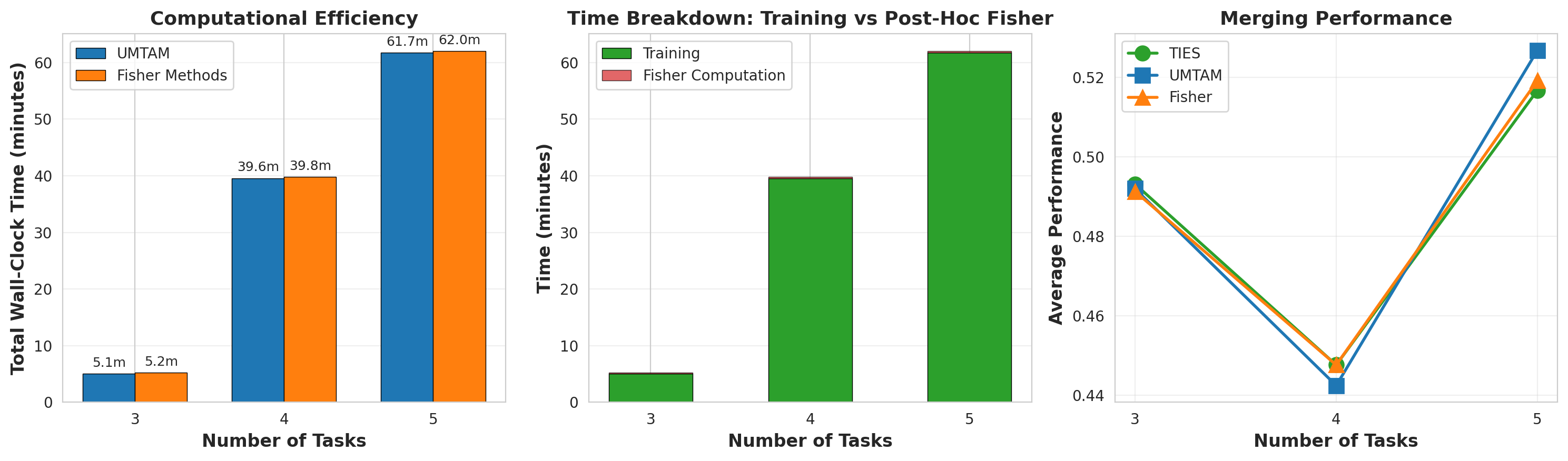}
\caption{Computational efficiency and merging performance across task counts. Left: Total wall-clock time comparison showing that UMTAM (blue) requires only training time, while Fisher methods (orange) incur additional post-hoc computation. Center: Time breakdown illustrating training time (green) and Fisher computation overhead (red). Right: Average merging performance across methods, showing that all approaches experience degradation in the 4-task configuration due to task interference, with UMTAM recovering most strongly in the 5-task scenario.}
\label{fig:computational_scaling}
\end{figure}

These results reveal a nuanced picture of curvature-aware merging that informs practical deployment decisions. The Fisher precision analysis demonstrates that post-hoc Fisher estimation with carefully chosen sample sizes can achieve competitive performance, but the non-monotonic relationship between sample count and merging quality indicates that Fisher-weighted methods require careful tuning of a hyperparameter that UMTAM eliminates entirely. By eliminating the need for a separate Fisher computation pipeline, UMTAM simplifies the practitioner's workflow---producing models that are ready for merging immediately upon training completion. This architectural simplification, combined with competitive or superior merging performance, positions UMTAM as a practical choice for multi-task model consolidation scenarios.

\subsection{Ablation Studies}
\label{sec:ablation}

To validate that each component of UMTAM contributes meaningfully to the overall framework---rather than merely combining existing techniques---we conduct systematic ablation experiments that isolate the contribution of each algorithmic element. This analysis addresses a natural concern: whether the performance gains arise from genuine synergy between components or could be achieved through simpler alternatives.

We evaluate five configurations that progressively remove UMTAM components, enabling direct measurement of each element's contribution to multi-task merging quality. The configurations are: Full UMTAM, which employs curvature-aware pruning, importance-weighted sign election, and curvature-weighted aggregation as described in Algorithm~\ref{alg:umtam_merging}; $-$Curvature Pruning, which replaces the saliency-based mask $\mathcal{I}_\tau^{(i,j)} = S_\tau^{(i,j)}$ with magnitude-based selection $|\Delta w_\tau^{(i,j)}|$, removing the curvature component from parameter importance estimation; $-$Sign Election, which omits the conflict resolution mechanism and includes all masked parameters regardless of sign disagreement; $-$Curvature Aggregation, which replaces the curvature-weighted averaging in Equation~\eqref{eq:merged-update} with uniform averaging across tasks; and Linear Averaging, the Task Arithmetic baseline~\citep{ilharco2022editing} that applies none of the UMTAM components, computing $w_{\text{merged}} = w_0 + \frac{1}{K}\sum_\tau \Delta w_\tau$.

All experiments merge four BERT-base experts trained on GLUE classification tasks (RTE, MRPC, QNLI, CoLA) from shared initialization, using the same training configuration as Section~\ref{sec:multitask_merging}. For methods requiring sparsity selection, we set $k = 20\%$ to enable direct comparison with our main results.

\begin{table}[H]
\centering
\caption{Component ablation study for UMTAM merging. Each row removes one component from the full framework. Metrics: accuracy for RTE and QNLI, F1 for MRPC, Matthews correlation for CoLA.}
\label{tab:ablation_results}
\renewcommand{\arraystretch}{1.15}
\setlength{\tabcolsep}{4pt}
\begin{tabular}{lcccccr}
\toprule
\textbf{Method} & \textbf{RTE} & \textbf{MRPC} & \textbf{QNLI} & \textbf{CoLA} & \textbf{Avg.} & \textbf{$\Delta$} \\
\midrule
Full UMTAM & 0.552 & \textbf{0.861} & \textbf{0.796} & 0.460 & \textbf{0.667} & --- \\
\midrule
$-$Curv.\ Pruning & 0.549 & 0.845 & 0.785 & 0.458 & 0.659 & $-$1.18\% \\
$-$Sign Election & \textbf{0.592} & 0.853 & 0.771 & 0.437 & 0.663 & $-$0.60\% \\
$-$Curv.\ Aggregation & 0.549 & 0.844 & 0.784 & \textbf{0.473} & 0.663 & $-$0.70\% \\
\midrule
Linear Averaging & 0.585 & 0.831 & 0.680 & 0.245 & 0.585 & $-$12.31\% \\
\bottomrule
\end{tabular}
\end{table}

Table~\ref{tab:ablation_results} presents the complete ablation results across all tasks and configurations. The Full UMTAM framework achieves the highest average score (0.667), outperforming every ablated variant. More importantly, removing any single component results in measurable performance degradation, confirming that each element contributes positively to the overall framework.

\begin{table}[H]
\centering
\caption{Individual contribution of each UMTAM component, measured as performance drop when the component is removed.}
\label{tab:component_contributions}
\renewcommand{\arraystretch}{1.15}
\begin{tabular}{lcc}
\toprule
\textbf{Component} & \textbf{Contribution} & \textbf{Significance} \\
\midrule
Curvature-Aware Pruning & +0.79\% & Primary \\
Curvature-Weighted Aggregation & +0.47\% & Secondary \\
Sign Election & +0.40\% & Secondary \\
\midrule
Sum of Individual & +1.66\% & --- \\
Total vs.\ Linear Avg. & +14.04\% & --- \\
\bottomrule
\end{tabular}
\end{table}

The component contributions, summarized in Table~\ref{tab:component_contributions}, reveal a clear hierarchy of importance. Curvature-aware pruning provides the largest individual contribution, with performance dropping by 1.18\% when replaced with magnitude-based selection. This finding validates our central hypothesis: the second-moment statistics accumulated during training offer a more principled importance signal than parameter magnitude alone. The curvature information captures not only which parameters changed substantially, but also which changes occur in regions of high loss sensitivity---precisely the parameters whose preservation matters most for maintaining task-specific performance.

Curvature-weighted aggregation and sign election provide secondary but meaningful contributions of 0.70\% and 0.60\% respectively. The aggregation component ensures that parameters in high-curvature regions---where small deviations from task optima incur large loss penalties---receive appropriate weighting during merging. Sign election resolves conflicts when multiple tasks push the same parameter in opposing directions, preventing destructive interference that would otherwise degrade the merged model.

A notable observation emerges from comparing individual contributions to the total improvement. The sum of individual component contributions (1.66\%) accounts for only a fraction of the 14.04\% improvement over Linear Averaging. This substantial gap indicates that the components interact synergistically: curvature-aware pruning identifies the right parameters to retain, sign election ensures these parameters receive coherent updates, and curvature-weighted aggregation balances their contributions appropriately. Removing any single component disrupts this coordinated mechanism, but removing all three---as in Linear Averaging---produces dramatically worse results than the sum of individual degradations would suggest.

The per-task results in Table~\ref{tab:ablation_results} reveal interesting patterns that align with our theoretical expectations. On semantic matching tasks (MRPC and QNLI), Full UMTAM achieves the best performance, suggesting that these tasks benefit from the complete curvature-aware framework. Two apparent anomalies merit discussion. First, removing sign election improves RTE performance (0.592 versus 0.552). This counterintuitive result suggests that textual entailment may exhibit minimal sign conflicts across the four merged tasks, such that the sign election mechanism provides limited benefit while potentially excluding some useful parameter updates. Second, removing curvature aggregation slightly improves CoLA performance (0.473 versus 0.460). This pattern echoes our earlier finding that syntactic judgment tasks may have flatter curvature landscapes where uniform aggregation proves equally effective. These task-specific variations do not undermine the overall framework; rather, they suggest opportunities for adaptive component weighting in future work.

\begin{figure}[H]
\centering
\begin{tikzpicture}
\begin{axis}[
    ybar,
    width=\columnwidth,
    height=0.55\columnwidth,
    ylabel={Average Score},
    symbolic x coords={Full UMTAM, $-$Curv.\ Prune, $-$Sign Elect., $-$Curv.\ Agg., Linear Avg.},
    xtick=data,
    x tick label style={rotate=15, anchor=east, font=\small},
    ymin=0.55, ymax=0.70,
    ytick={0.55, 0.58, 0.61, 0.64, 0.67, 0.70},
    bar width=0.6cm,
    nodes near coords,
    nodes near coords style={font=\footnotesize, above, yshift=1pt},
    every node near coord/.append style={/pgf/number format/.cd, fixed, precision=3},
    grid=major,
    grid style={dashed, gray!30},
    legend style={at={(0.98,0.95)}, anchor=north east, font=\small},
]
\addplot[fill=umtam, draw=black, line width=0.5pt] coordinates {
    (Full UMTAM, 0.667)
    ($-$Curv.\ Prune, 0.659)
    ($-$Sign Elect., 0.663)
    ($-$Curv.\ Agg., 0.663)
    (Linear Avg., 0.585)
};
\draw[dashed, line width=1pt, umtam!70] (axis cs:Full UMTAM,0.667) -- (axis cs:Linear Avg.,0.667);
\end{axis}
\end{tikzpicture}
\caption{Impact of removing each UMTAM component on average merging performance. The dashed line indicates Full UMTAM performance. The progressive degradation from Full UMTAM through the ablated variants to Linear Averaging illustrates the cumulative value of the framework's components.}
\label{fig:ablation_impact}
\end{figure}
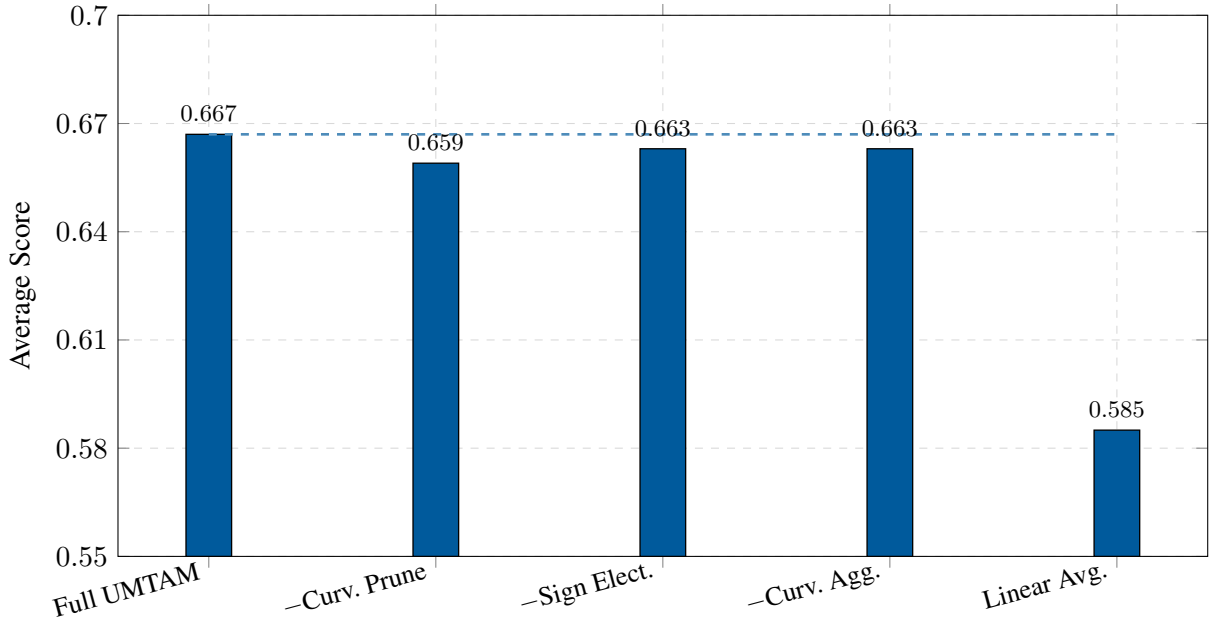

The ablation study yields three principal conclusions. First, each UMTAM component contributes positively to merging quality, validating that the framework is not merely an aggregation of existing techniques but a coherent system where components reinforce each other. Second, curvature-aware pruning---the component that most directly leverages training-time statistics---provides the largest individual contribution, supporting our core claim that optimization trajectory information enables more effective parameter selection than magnitude alone. Third, the synergistic interaction between components produces improvements far exceeding the sum of individual contributions, demonstrating that the unified training-merging framework captures genuine complementarity between curvature-aware parameter selection, conflict resolution, and geometry-respecting aggregation.

\subsection{Instruction Tuning at Scale}
\label{sec:instruction_tuning}

The experiments presented thus far validate UMTAM on classification tasks with BERT-scale models. To demonstrate compatibility with modern efficient training practices and larger architectures, we conduct instruction tuning experiments on the Tulu-3-sft-mixture dataset~\citep{lambert2024tulu}, which comprises diverse instruction-following examples spanning reasoning, coding, and general knowledge. We fine-tune Mistral-7B on a 5,000-sample subset using LoRA adaptation with rank 64, targeting all attention and MLP projection matrices. Both UMTAM and AdamW employ identical configurations: learning rate $2 \times 10^{-5}$, batch size 2 with gradient accumulation 8, sequence length 1024, and single-epoch training with 4-bit quantization to enable training on a single A100 GPU.

Table~\ref{tab:tulu3} presents the results. UMTAM achieves validation loss of 0.761 compared to AdamW's 0.762, with both methods converging to identical perplexity of 2.14. This equivalence is the expected and desired outcome: UMTAM extends AdamW by maintaining factorized curvature statistics and accumulating saliency scores, but these additions should not alter optimization dynamics when properly implemented. The results confirm that practitioners can substitute UMTAM for AdamW without sacrificing training performance.

\begin{table}[H]
\centering
\caption{Instruction tuning on Tulu-3 (5K samples) with Mistral-7B.}
\label{tab:tulu3}
\renewcommand{\arraystretch}{1.15}
\begin{tabular}{lccc}
\toprule
\textbf{Optimizer} & \textbf{Val Loss} & \textbf{Perplexity} & \textbf{Curvature Tracking} \\
\midrule
AdamW & 0.762 & 2.14 & $\times$ \\
UMTAM & \textbf{0.761} & \textbf{2.14} & $\checkmark$ \\
\bottomrule
\end{tabular}
\end{table}

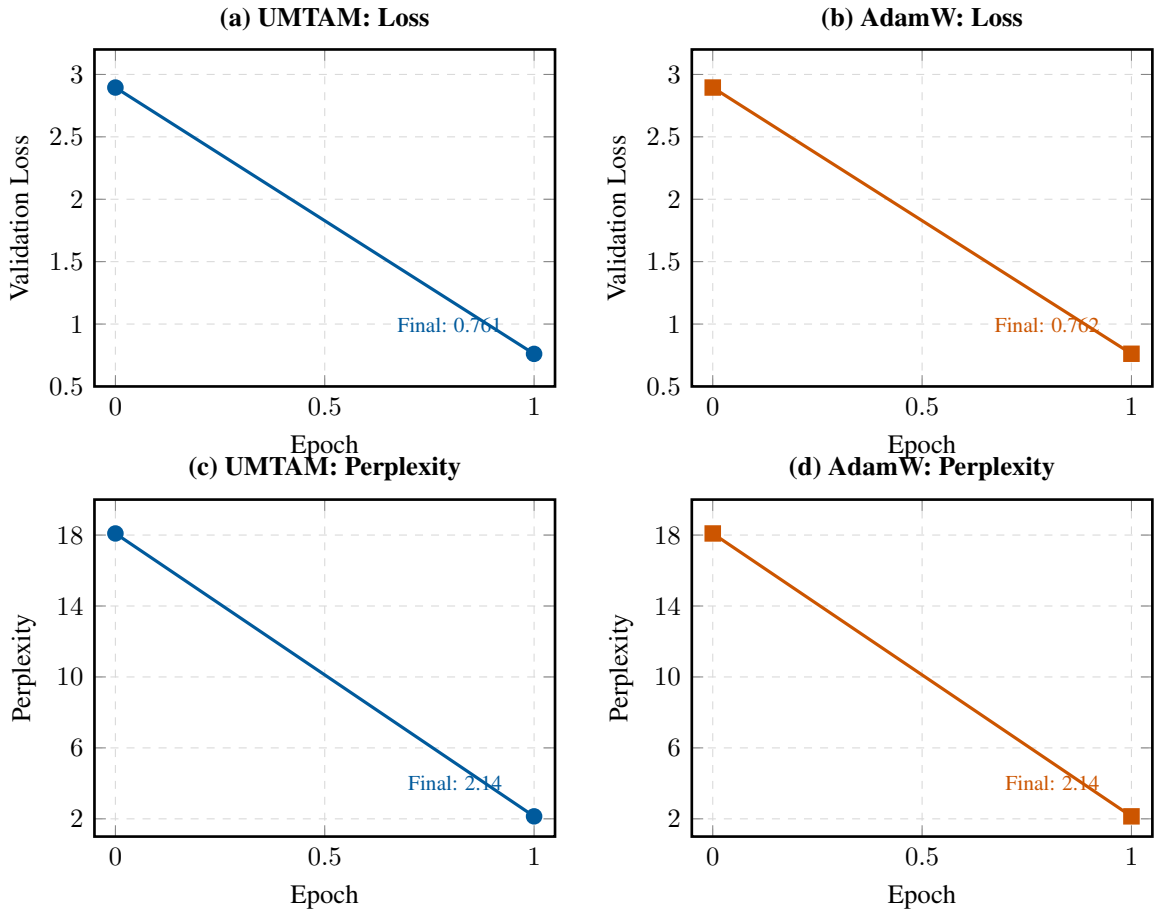
\begin{figure}[H]
\centering
\begin{tikzpicture}
\begin{groupplot}[
    group style={
        group size=2 by 2,
        horizontal sep=1.8cm,
        vertical sep=1.5cm,
    },
    width=0.48\columnwidth,
    height=0.38\columnwidth,
    grid=major,
    grid style={dashed, gray!30},
    tick label style={font=\small},
    label style={font=\small},
    title style={font=\small\bfseries, yshift=-1mm},
]

\nextgroupplot[
    xlabel={Epoch},
    ylabel={Validation Loss},
    xmin=-0.05, xmax=1.05,
    ymin=0.5, ymax=3.2,
    xtick={0, 0.5, 1.0},
    ytick={0.5, 1.0, 1.5, 2.0, 2.5, 3.0},
    title={(a) UMTAM: Loss},
    mark size=2.5pt,
    line width=1pt,
]
\addplot[color=umtam, mark=*, solid, line width=1.2pt] coordinates {
    (0, 2.895) (1, 0.761)
};
\node[anchor=south east, font=\scriptsize, umtam] at (axis cs:0.95, 0.85) {Final: 0.761};

\nextgroupplot[
    xlabel={Epoch},
    ylabel={Validation Loss},
    xmin=-0.05, xmax=1.05,
    ymin=0.5, ymax=3.2,
    xtick={0, 0.5, 1.0},
    ytick={0.5, 1.0, 1.5, 2.0, 2.5, 3.0},
    title={(b) AdamW: Loss},
    mark size=2.5pt,
    line width=1pt,
]
\addplot[color=mofasgd, mark=square*, solid, line width=1.2pt] coordinates {
    (0, 2.895) (1, 0.762)
};
\node[anchor=south east, font=\scriptsize, mofasgd] at (axis cs:0.95, 0.85) {Final: 0.762};

\nextgroupplot[
    xlabel={Epoch},
    ylabel={Perplexity},
    xmin=-0.05, xmax=1.05,
    ymin=1, ymax=20,
    xtick={0, 0.5, 1.0},
    ytick={2, 6, 10, 14, 18},
    title={(c) UMTAM: Perplexity},
    mark size=2.5pt,
    line width=1pt,
]
\addplot[color=umtam, mark=*, solid, line width=1.2pt] coordinates {
    (0, 18.09) (1, 2.14)
};
\node[anchor=south east, font=\scriptsize, umtam] at (axis cs:0.95, 3.0) {Final: 2.14};

\nextgroupplot[
    xlabel={Epoch},
    ylabel={Perplexity},
    xmin=-0.05, xmax=1.05,
    ymin=1, ymax=20,
    xtick={0, 0.5, 1.0},
    ytick={2, 6, 10, 14, 18},
    title={(d) AdamW: Perplexity},
    mark size=2.5pt,
    line width=1pt,
]
\addplot[color=mofasgd, mark=square*, solid, line width=1.2pt] coordinates {
    (0, 18.09) (1, 2.14)
};
\node[anchor=south east, font=\scriptsize, mofasgd] at (axis cs:0.95, 3.0) {Final: 2.14};

\end{groupplot}
\end{tikzpicture}
\caption{Convergence on Tulu-3 instruction tuning. Top row: validation loss for UMTAM (a) and AdamW (b). Bottom row: perplexity for UMTAM (c) and AdamW (d). Both optimizers achieve identical final performance, confirming that curvature tracking introduces no optimization penalty.}
\label{fig:tulu3_convergence}
\end{figure}

Figure~\ref{fig:tulu3_convergence} displays the convergence trajectories for each optimizer. Both methods reduce validation loss from 2.90 to 0.76 and perplexity from 18.1 to 2.14 over one epoch, with virtually indistinguishable curves confirming that curvature tracking introduces no optimization overhead.

The value of UMTAM emerges in downstream applications rather than raw training metrics. While AdamW discards all optimizer state upon completion, UMTAM preserves curvature information enabling principled model merging. For instruction-tuned models, this addresses a practical scenario: organizations often fine-tune foundation models for specialized domains---legal, medical, coding---and subsequently wish to combine capabilities without retraining. UMTAM's accumulated statistics provide importance weights for curvature-aware merging, eliminating redundant post-hoc Fisher computation.

The experiment also validates compatibility with modern efficient training practices: 4-bit quantization, LoRA adaptation, and gradient checkpointing integrate seamlessly with UMTAM, requiring no modifications. Peak memory of 21.6 GB confirms feasibility on consumer hardware. A limitation is the reduced dataset size (5K versus 939K samples), necessitated by computational constraints but sufficient to demonstrate training equivalence and the potential for downstream merging applications.

\section{Conclusion}
\label{sec:conclusion}

This paper introduced UMTAM, a framework that unifies memory-efficient training and model merging through shared computational structure. The core insight is simple but consequential: the curvature information that adaptive optimizers compute and discard is precisely what merging methods must later recompute. By preserving this information throughout training, UMTAM eliminates redundant computation while enabling more principled model composition.

Our theoretical analysis establishes that this unification does not sacrifice formal guarantees. UMTAM achieves $O(1/\sqrt{T})$ convergence for non-convex objectives despite dual momentum tracking, with approximation error bounded by singular value decay of the gradient. Empirical spectral analysis validates these assumptions: rank-32 factorization captures over 82\% of momentum energy throughout training, with stable rank hovering around 3---explaining why UMTAM exhibits rank-invariant convergence across configurations spanning an order of magnitude.

Our experiments validate three claims. First, curvature-aware saliency provides superior parameter importance estimation---outperforming magnitude-only selection by up to 135\% at aggressive sparsity levels, with effectiveness determined by whether task-critical knowledge concentrates in localized high-curvature regions. Second, UMTAM matches or exceeds the training performance of specialized low-rank optimizers while exhibiting hyperparameter robustness that simplifies practical deployment. Third, the unified framework produces merged models that recover 96\% of individual expert performance, improving 14.9\% over linear averaging and 1.6\% over strong baselines like TIES. Comparison with post-hoc Fisher estimation confirms that trajectory-aware curvature provides importance estimates comparable to or better than recomputed Fisher information, while eliminating the separate computation pass entirely.

The framework integrates seamlessly with modern efficient training practices. Instruction tuning experiments on Mistral-7B with LoRA and 4-bit quantization demonstrate that UMTAM matches AdamW performance while accumulating curvature statistics for downstream merging---requiring no modifications to established workflows. Our evaluation focuses on language modeling tasks, consistent with related work on low-rank optimizers; validation on vision transformers and multimodal architectures remains a promising direction for future research.

Beyond empirical results, UMTAM offers a conceptual contribution: it reframes the optimization trajectory as a valuable asset rather than intermediate computation to be discarded. This perspective opens several directions for future work. The framework naturally extends to continual learning, where curvature statistics could guide selective plasticity. Federated settings could leverage trajectory information for communication-efficient model aggregation. And the connection between training dynamics and merging quality suggests that optimizers might be designed with downstream composition explicitly in mind.

The practical appeal of UMTAM lies in what it does not require: no post-hoc Fisher computation, no separate importance estimation pass, no careful coupling of rank and learning rate. Practitioners can substitute UMTAM for standard optimizers, train within comparable memory budgets (approximately 30\% overhead for curvature tracking), and obtain models ready for principled multi-task composition. As foundation models proliferate and task-specific adaptation becomes routine, the ability to efficiently combine specialized capabilities grows increasingly valuable. UMTAM provides a principled path forward: train once, merge intelligently, deploy anywhere.

\bibliographystyle{unsrtnat}
\bibliography{references}

\end{document}